\documentclass{article}

\usepackage[preprint,main]{neurips_2026}

\usepackage[utf8]{inputenc}
\usepackage[T1]{fontenc}
\usepackage{url}

\usepackage{amsmath,amssymb,amsthm}
\usepackage{mathtools}

\usepackage{graphicx}
\usepackage{booktabs}
\usepackage{nicefrac}
\usepackage{microtype}
\usepackage{xcolor}
\usepackage{placeins}

\usepackage{hyperref}

\theoremstyle{plain}
\newtheorem{theorem}{Theorem}
\newtheorem{lemma}{Lemma}

\theoremstyle{definition}
\newtheorem{definition}{Definition}

\theoremstyle{remark}
\newtheorem{remark}{Remark}

\newcommand{\vz}{\mathbf{z}}
\newcommand{\ve}{\mathbf{e}}
\newcommand{\valpha}{\boldsymbol{\alpha}}
\newcommand{\vpi}{\boldsymbol{\pi}}
\newcommand{\vp}{\mathbf{p}}
\newcommand{\vy}{\mathbf{y}}

\newcommand{\vtheta}{\boldsymbol{\theta}}
\newcommand{\dgam}{\psi^{(0)}}

\newcommand{\EE}{\mathbb{E}}
\newcommand{\PP}{\mathbb{P}}
\newcommand{\R}{\mathbb{R}}

\DeclareMathOperator{\Dir}{Dir}
\DeclareMathOperator{\KL}{KL}
\DeclareMathOperator*{\argmax}{arg\,max}

\newcommand{\Piop}{\Pi}

\newcommand{\norm}[1]{\left\lVert #1 \right\rVert}
\newcommand{\opnorm}[1]{\left\lVert #1 \right\rVert_{\mathrm{op}}}

\newcommand{\alphazero}{\alpha_0}

\title{Plug-in Losses for Evidential Deep Learning: A Simplified Framework for Uncertainty Estimation that Includes the Softmax Classifier}

\author{
   Berk Hayta \\
  TU Munich \\
   \ \texttt{berkhayta.contact@gmail.com} \\
  \And  
  Hannah Laus \\
   TU Munich \& MCML  \\
  \And
  \qquad \qquad Simon Mittermaier \\
  \qquad \qquad Infineon Technologies \\
  \And
  \qquad Felix Krahmer \\
  \qquad TU Darmstadt, TU Munich \& MCML 
}

\begin{document}
\maketitle

\begin{abstract}
Real-world sensor-based learning systems require uncertainty estimation that is both reliable and computationally efficient.
Evidential Deep Learning (EDL) provides single-pass uncertainty estimation
by modeling the class probabilities via Dirichlet distributions, where the
Dirichlet parameters are predicted by a learned neural network mapping. However, this approach can lead to computational challenges, as
Dirichlet expected objectives are more complex than standard supervised learning 
losses, complicating their analysis and implementation. We address this issue by approximating the objective of the first-order empirical risk minimization problem induced by EDL with a plug-in loss
evaluated at the Dirichlet mean and show that, under mild assumptions, the approximation error decays with growing evidence for a broad class of loss functions, including mean-squared error and cross-entropy loss. 
As a special case, our analysis provides justification for the use of softmax in the context of uncertainty estimation, since under a particular evidence-to-Dirichlet mapping, our framework includes the standard softmax classifier.
We validate the proposed simplified objectives on the Google Speech Commands dataset and show that they achieve predictive accuracy
and selective prediction performance comparable to classical EDL, while
being simpler to implement using standard deep learning losses and
training pipelines.  To the best of our knowledge, this empirical analysis is the first to obtain coverage-accuracy trade-offs for speech recognition tasks through EDL.
\end{abstract}

\section{Introduction}
Deep neural networks are increasingly deployed in resource-constrained and
safety-critical settings, where predictive accuracy alone is often
insufficient. Overconfident predictions may lead to undesirable or unsafe
system behavior, motivating methods that quantify predictive uncertainty
and enable selective prediction. At the same time, many established approaches for uncertainty estimation come with significant computational overhead, e.g., Monte Carlo dropout (\cite{gal2016dropout}), Bayesian neural networks (\cite{neal2012bayesian,blundell2015weight}) and Deep ensembles (\cite{lakshminarayanan2017simple,wilson2020bayesian}). 
This can make them unsuitable for real-time applications such as embedded
sensing, where uncertainty estimation methods should add only minimal
inference-time overhead; moreover, many of these methods require specific
training procedures, making them less straightforward to integrate into
existing pipelines.

Evidential Deep Learning (EDL) (\cite{Sensoy}) aims to address these issues by enabling
uncertainty estimation in a single forward pass while remaining
compatible with standard neural network training. The idea of EDL is not to learn only a single prediction, but rather to predict parameters of a Dirichlet distribution modeling the class probabilities.

Despite promising
empirical results for controlled benchmark settings (\cite{Sensoy,deng2023uncertainty}), there remains a gap between the theoretical foundations of this uncertainty-aware learning strategy and the algorithmic implementation. 
In particular, the theory-based loss function of EDL is based on an explicit formula for the 
expected loss over the unknown Dirichlet distribution to be estimated. 

Compared with ordinary supervised losses, the resulting
objectives are more complex because they depend on both the projected class
probabilities and the total Dirichlet concentration. This additional structure
can make the optimization landscape harder to interpret and may contribute to
sensitivity in hyperparameter choices, slow convergence, or convergence to poor
local minima. This motivates a closer analysis of the empirical-risk structure
induced by EDL and the plug-in simplifications developed below.

In this work, we propose a remedy for these concerns by working with an approximation to the EDL loss function that is closer to the loss without EDL, while also retaining the uncertainty information provided by its minimizer.

Thus, the resulting uncertainty-aware learning problem can be
optimized using simpler supervised-learning objectives with efficient implementations and well-understood training behavior.

To illustrate this, we validate the resulting simplified objectives on the Google Speech Commands v1 dataset (\cite{warden2018speechcommands}).  As we demonstrate in theoretical derivations and numerical studies, the resulting uncertainty-aware learning methods exhibit comparable operational metrics to EDL.

More precisely, we proceed via a Taylor expansion of the Dirichlet expected loss around the Dirichlet mean. The first-order term vanishes in expectation, and higher terms decay with growing evidence, which establishes that evaluating the loss at the Dirichlet mean yields a good approximation of the EDL objective. As a special case, the softmax classifier belongs to this class of simplified evidential classifiers.

For the Google Speech Commands v1 dataset, we estimate uncertainty
distributions for correct and incorrect predictions and evaluate the resulting
coverage--accuracy trade-off. By thresholding an uncertainty score, the
classifier can withhold predictions for uncertain samples, increasing accuracy
on the accepted samples at the cost of reduced coverage. Both EDL and the
proposed approximate EDL objectives, including the softmax case, provide such
uncertainty scores.

\paragraph{Contributions.}
The contributions of this work are as follows:

\begin{itemize}

 \item In this paper, we derive that the evidential deep learning framework (EDL) can be simplified by working with plug-in losses that approximate the EDL objective proposed in \cite{Sensoy}. 
 
 \item We rigorously show that, for growing evidence, the EDL objective is well approximated by the proposed plug-in loss and quantify the approximation error.
 \item  The resulting classification framework includes a variant of the classical softmax classifier, hence our analysis provides a first step towards understanding the performance of softmax for uncertainty estimation in the EDL framework.

\item We apply the resulting uncertainty estimation framework for a realistic keyword spotting task on Google Speech Commands v1 and find that it performs comparably to the original EDL formulation, despite its simpler implementability. 

\item We illustrate the benefits of the resulting uncertainty estimates by evaluating operational uncertainty through fixed
high-reliability operating points, coverage--accuracy trade-offs, and
uncertainty distributions for correct and incorrect predictions.  To the best of our knowledge, this is the first uncertainty analysis in the spirit of EDL in the speech recognition context.

\end{itemize}
\section{Related work}

\paragraph{Evidential Deep Learning}

Evidential deep learning was introduced by \cite{Sensoy} for classification and is inspired by Dempster-Shafer theory (\cite{dempster1968generalization,shafer1990perspectives,shafer1976mathematical}) and subjective logic (\cite{josang2001logic,jsang2018subjective}).
In recent years many refinements and extensions of EDL have been proposed (\cite{ulmer2023prior, gao2025edl_survey}). For example,
R-EDL (\cite{R_EDL}) replaces the fixed prior-weight choice in the
subjective-logic parametrization of classical EDL with a tunable
hyperparameter, and separately simplifies the original EDL mean-squared-error
objective by removing its variance term. I-EDL (\cite{deng2023uncertainty}) adds an extra Fisher regularization term to the loss.

Re-EDL (\cite{Re_EDL}) and \cite{EDL_Mirage} both argue that for improved out-of-distribution detection one can leave out the KL-divergence term. 

Besides advancements in EDL, there is also a line of work criticizing EDL
(\cite{bengs2022pitfalls,jurgens2024epistemic,EDL_Mirage}).
\cite{jurgens2024epistemic} shows that EDL does not model epistemic
uncertainty in the formal Bayesian sense, since its uncertainty does not
necessarily vanish with infinite training data and is highly dependent on
hyperparameter choices. On the other hand, \cite{EDL_Mirage} shows that EDL can
perform well for out-of-distribution detection, even if it should not be
interpreted as Bayesian uncertainty quantification. These observations motivate
a direct study of the Dirichlet-expected losses used in classification EDL. In
particular, our work analyzes the relation between losses evaluated under the
Dirichlet expectation and losses evaluated at the Dirichlet mean.

In a parallel line of work, \cite{malinin2018predictive} introduced Prior Networks which also use Dirichlet-output classifiers. The method of \cite{malinin2018predictive} requires OOD samples for training which is not available in many applications but other works in that line also get away without OOD samples (\cite{tsiligkaridis2021information,haussmann2020bayesian}). Furthermore, there are works on EDL for regression (\cite{amini2020deep,malinin2020regression}).
In this work, we focus on classification EDL as introduced in \cite{Sensoy},
while noting that related Dirichlet-output classifiers have also been studied
in the Prior and Posterior Network literature (\cite{ulmer2023prior}).

\paragraph{Softmax confidence and deterministic uncertainty}
\cite{holm2023revisiting} analyze the empirical uncertainty-estimation properties of softmax confidence compared to MC dropout for text classification and find that it performs competitively. In contrast, \cite{ovadia2019can} presents a large-scale empirical study showing that deterministic confidence measures, such as softmax, fail on certain tasks under distribution shifts and can display misleading calibration behavior.

\paragraph{Speech command recognition}
Speech Command Recognition, often also called Keyword Spotting or Wake-Word Detection, is a subtask within the field of Speech Recognition with a focus on detecting certain words or phrases. The most common benchmark for this task is derived from the Google Speech Commands v1 dataset, described in \cite{warden2018speechcommands}. In practice, this task is often performed with limited compute and power resources, often in battery-powered settings.
\cite{majumdar2020matchboxnet} presents a model architecture optimized not only for prediction accuracy but also for efficient inference. We use this model as the basis for our experiments and, to the best of our knowledge, are the first to showcase evidential deep learning in this realistic and impactful application.

\section{Preliminaries}
\subsection{Problem setup}

We consider a supervised multiclass classification problem with $K$
classes. Let $\mathcal{X} \subseteq \R^d$ denote the input space
and let $\mathcal{Y} = \{1,\dots,K\}$ denote the label space. Data are
drawn from an unknown distribution $\PP_{XY}$ over
$\mathcal{X}\times\mathcal{Y}$, and we observe an i.i.d.\ sample
$\{(x_i,y_i)\}_{i=1}^n$.
Let $f_{\vtheta}:\mathcal{X}\to\R^K$ be a parametric model with
parameters $\vtheta\in\Theta$, and let
$\vz=f_{\vtheta}(x)\in\R^K$ denote the output logits. Predictions are
obtained by composing $f_{\vtheta}$ with a mapping
$\hat{\vp}:\R^K\to\Delta^{K-1}$ into the probability simplex.
Throughout, class labels may be identified with their one-hot
representations when convenient. 
Given a loss function
$
\ell:\Delta^{K-1}\times\mathcal{Y}\to\R_+,
$
the population and empirical risks are defined as
\begin{gather}
\mathcal R(\vtheta)
=
\EE_{(X,Y)\sim\PP_{XY}}
\big[
\ell(\hat{\vp}(f_{\vtheta}(X)),Y)
\big],
\qquad
\widehat{\mathcal R}_n(\vtheta)
=
\frac1n\sum_{i=1}^n
\ell(\hat{\vp}(f_{\vtheta}(x_i)),y_i).
\end{gather}
The goal of a learning problem is to minimize the empirical risk, i.e.,
$
\vtheta^\star \in
\arg\min_{\vtheta\in\Theta}
\widehat{\mathcal R}_n(\vtheta),
$
which means that model predictions are learned by minimizing a loss over the training data. This is called empirical risk minimization (\cite[Chapter 4.2]{mohri2018foundations}).

\subsection{Evidential Deep Learning}
\label{sec:edl}

We consider the Evidential Deep Learning (EDL) framework for
classification introduced by \cite{Sensoy}. Throughout, we refer to the
original formulation as \emph{classical EDL} to distinguish it from the
simplified formulations introduced later.

For a loss \(\ell:\Delta^{K-1}\times\mathcal Y\to\R_+\), the classical
Dirichlet-expected EDL objective is
\begin{gather}
\label{eq:edl_objective}
\EE_{(X,Y)\sim\PP_{XY}}
\Big[
\EE_{\vpi \sim \Dir(\valpha_{\vtheta}(X))}
\big[
\ell(\vpi,Y)
\big]
\Big].
\end{gather}
The neural network enters this objective through the Dirichlet-parameter map
\(x\mapsto \valpha_{\vtheta}(x)\). Concretely, it outputs logits
$\vz=f_{\vtheta}(x)\in\R^K$, which are mapped to componentwise evidence values
$\ve=\tau(\vz)$, where $\tau$ is typically monotone increasing
(e.g., exponential or softplus function).
The evidence is then mapped to Dirichlet parameters
\(\valpha=\phi(\ve)\). The classical choice is
\(\alpha_i=e_i+1\), while we allow componentwise maps
\(\phi:\R_{\ge0}^K\to\R_{>0}^K\).  Predictions are obtained from the Dirichlet
mean
\begin{gather}
\hat{\vp}(\vz)=\Piop(\valpha(\vz))
=\frac{\valpha(\vz)}{\alpha_0(\vz)},
\qquad
\alpha_0(\vz)=\sum_{j=1}^K\alpha_j(\vz).
\end{gather}
The hard decision satisfies
\(
\hat y(\vz)=\argmax_i \hat p_i(\vz)
=\argmax_i \alpha_i(\vz)
=\argmax_i e_i(\vz).
\)

Since classical EDL is usually specified through the training objective, we
separate out the underlying logits-to-Dirichlet-parameter map and call this
map the classical evidential classifier.
\begin{definition}[Classical evidential classifier]
\label{def:classical-evidential-classifier}
Using the mappings above, a classical evidential classifier is the logits-to-Dirichlet-parameter map
\begin{gather}
\vz \mapsto \valpha(\vz)=\phi(\tau(\vz)).
\end{gather}
\end{definition}
We use two classical EDL losses from \cite{Sensoy}. Let
\(\vy\in\{0,1\}^K\) denote the one-hot encoded target label. The
cross-entropy objective is written as the Dirichlet expectation
\begin{gather}
\label{eq:CE_EDL_loss}
L_{\mathrm{CE}}^{\mathrm{EDL}}(\valpha,\vy)
=
\EE_{\vpi\sim\Dir(\valpha)}
\!\left[
-\sum_{k=1}^K y_k \log \pi_k
\right]
=
\sum_{k=1}^K y_k
\left(
\dgam(\alpha_0)-\dgam(\alpha_k)
\right),
\end{gather}
where $\dgam$ denotes the digamma function. The evidential mean-squared-error
loss is
\begin{gather}
\label{eq:MSE_EDL_loss}
L_{\mathrm{MSE}}^{\mathrm{EDL}}(\valpha,\vy)
=
\sum_{k=1}^K
\left(y_k-\hat p_k\right)^2
+
\sum_{k=1}^K
\frac{\hat p_k(1-\hat p_k)}{\alpha_0+1}.
\end{gather}

Classical EDL is commonly trained with an additional KL regularizer that
penalizes evidence assigned to incorrect classes. Following \cite{Sensoy}, the
regularized loss is
\begin{gather}
\ell_{\mathrm{EDL}}^{\mathrm{reg}}(\valpha,\vy)
=
\ell_{\mathrm{EDL}}(\valpha,\vy)
+
\lambda_t
\KL\!\left(
\Dir(\tilde{\valpha})
\,\middle\|\,
\Dir(\mathbf{1})
\right),
\qquad
\tilde{\valpha}=\vy+(1-\vy)\odot\valpha.
\end{gather}
\section{Simplified Evidential Deep Learning}
\label{sec:simplified_edl}
\subsection{A first-order ERM view of classical EDL}
\label{sec:deterministic_erm}
Before simplifying classical EDL, we first show that it can be viewed within the standard first-order ERM framework.
Classical EDL objectives are commonly written as nested expectations.
However, the inner Dirichlet distribution is fully determined by the
model output through the mapping \(\vz \mapsto \valpha(\vz)\). In practical
implementations, the inner expectation $\EE_{\vpi \sim \Dir(\valpha(\vz))}
\big[
\ell(\vpi,\vy)
\big]$ is reduced to a closed-form
expression before training (e.g. Eq. \ref{eq:CE_EDL_loss} or Eq. \ref{eq:MSE_EDL_loss}), yielding the scalar loss used for
gradient-based optimization.
Accordingly, the Dirichlet expectation defines a deterministic
transformation of the base loss \(\ell\), yielding the induced loss
\(\ell_{\mathrm{EDL}}\) on the model output.

\begin{definition}[Induced evidential loss]
Let $\ell : \Delta^{K-1} \times \mathcal{Y} \to \R_+$ be a base loss. Then the induced evidential loss associated with classical EDL is defined by
\begin{gather}
\label{eq:edl_deterministic_loss}
\ell_{\mathrm{EDL}}(\vz,\vy)
:=
\EE_{\vpi \sim \Dir(\valpha(\vz))}
\big[
\ell(\vpi,\vy)
\big].
\end{gather}
\end{definition}

Because the Dirichlet parameters are fully determined by the model
output, optimization remains over the model parameters \(\vtheta\)
alone. The inner expectation therefore does not introduce an additional
optimization variable, but instead defines the scalar loss
\(\ell_{\mathrm{EDL}}\) used in empirical risk minimization.
Classical EDL can therefore be analyzed within the standard first-order
empirical risk minimization framework using the induced loss
\(\ell_{\mathrm{EDL}}\).
\subsection{Simplifying Dirichlet-expected evidential losses}
\label{sec:losses}
Following the first-order ERM view in Section~\ref{sec:deterministic_erm},
we study the per-sample losses induced by the inner Dirichlet expectation
and show that they are well approximated by plug-in losses evaluated at the
projected probabilities.

\begin{definition}[Plug-in evidential loss]
\label{def:plugin-evidential-loss}
Let
$\ell : \Delta^{K-1}\times\mathcal{Y}\to\R_+ $
be a classification loss on the probability simplex. The corresponding plug-in evidential loss is defined by
\begin{gather}
\ell_{\mathrm{plug}}(\valpha,\vy)
=
\ell(\Piop(\valpha),\vy),
\qquad
\Piop(\valpha)=\frac{\valpha}{\alphazero},
\qquad
\alphazero=\sum_{j=1}^K \alpha_j.
\end{gather}
\end{definition}
Thus, the parameters \(\valpha\) affect the objective only through the
projected probabilities \(\Piop(\valpha)\). Consequently, for a parametric model \(x\mapsto \valpha_{\vtheta}(x)\), the resulting expected-risk objective is
\begin{gather}
\EE_{(X,Y)\sim\PP_{XY}}
\big[
\ell(\Piop(\valpha_{\vtheta}(X)),Y)
\big],
\end{gather}
replacing the Dirichlet-expected loss used in classical EDL.
For example we can write the cross-entropy loss and the mean squared error loss in its plug-in formulation
\begin{gather}
\ell^{\mathrm{CE}}_{\mathrm{plug}}(\valpha,\vy)
=
-\log \Piop(\valpha)_y,
\qquad
\ell^{\mathrm{MSE}}_{\mathrm{plug}}(\valpha,\vy)
=
\|\Piop(\valpha)-\vy\|_2^2.
\end{gather}
The plug-in cross-entropy objective also appears in \cite{Sensoy}
through a Type-II maximum likelihood derivation, while the plug-in MSE
corresponds to the classical EDL MSE objective (Eq. \ref{eq:MSE_EDL_loss}) without its Dirichlet variance
term, as also considered in R-EDL (\cite{R_EDL}). In our formulation, both arise from the
same construction: replacing the Dirichlet-expected loss by
a classification loss on the probability simplex evaluated at the projected probabilities. More generally, any loss on the probability simplex induces a simplified evidential objective via
composition with \(\Piop\), avoiding explicit Dirichlet expectations while
preserving the classifier output \(\Piop(\valpha)\).

\subsection{Simplified evidential classifiers}
\label{sec:simplified_classifiers}
To make predictions explicit, we define a simplified evidential classifier
whose decisions depend only on \(\hat{\vp}(\vz)\), in contrast to the
classical evidential classifier in
Definition~\ref{def:classical-evidential-classifier}.

\begin{definition}[Simplified evidential classifier]
\label{def:simplified-evidential-classifier}
A \emph{simplified evidential classifier} is the induced probability map
\begin{gather}
\vz \mapsto \hat{\vp}(\vz)=\Piop(\valpha(\vz))
=
\Piop(\phi(\tau(\vz))),
\end{gather}
used with learning objectives that depend only on
\(\hat{\vp}(\vz)\), rather than directly on the intermediate parameter vector
\(\valpha(\vz)\).
\end{definition}

\paragraph{Interpretation.}
The distinction between Definitions \ref{def:classical-evidential-classifier} and \ref{def:simplified-evidential-classifier} is not the predictive map
itself, but the role assigned to the intermediate vector \(\valpha\)
during training. In classical EDL, \(\valpha\) enters the objective
through Dirichlet-based losses or regularization terms. Under plug-in
objectives, \(\valpha\) acts only as a deterministic parametrization whose
projection \(\hat{\vp}(\vz)\) determines both prediction and optimization.
Thus, for plug-in objectives, the optimization problem can be
studied through the induced probability map without relying on the full
Dirichlet-expected loss structure.

\begin{remark}
Our definition of simplified evidential classifiers allows the additive
constant \(c\ge 0\). The case \(c=0\) is included intentionally. Since the
evidence map satisfies \(e_i>0\), the parameters \(\alpha_i=e_i\) still define
a valid Dirichlet distribution.
This case should not be confused with the Subjective Logic interpretation
associated with classical EDL (~\cite{Sensoy,R_EDL}). Classical EDL uses
\(\alpha_i=e_i+1\), for which \(K/\alpha_0\) corresponds to the usual vacuity
mass. When \(c=0\), the fixed additive prior mass is absent, and
\(K/\alpha_0\) is no longer guaranteed to be bounded by one. We therefore
interpret \(K/\alpha_0\) for \(c=0\) as an inverse-concentration diagnostic
rather than as standard Subjective-Logic vacuity.
\end{remark}

\subsection{Softmax models as Simplified evidential classifiers}
\label{sec:softmax}

Simplified evidential classifiers are defined through their
logits-to-probabilities map. More broadly, any output layer that maps logits
to positive class scores and then normalizes them fits this definition. The
softmax classifier is a canonical example.

\begin{theorem}[Softmax as a simplified evidential classifier]
\label{thm:softmax_simplified}
Any neural network equipped with a softmax output layer defines a
simplified evidential classifier in the sense of
Definition~\ref{def:simplified-evidential-classifier}.
\end{theorem}

\begin{proof}
Let \(\vz\in\R^K\) denote the network logits. Choose
\(\tau(\vz)=\exp(\vz)\) and \(\phi(\ve)=\ve\). Then \(\alpha_i(\vz)=\exp(z_i)\), and therefore
\(
\Piop(\valpha(\vz))_i
=
\frac{\alpha_i(\vz)}{\sum_{j=1}^K \alpha_j(\vz)}
=
\frac{\exp(z_i)}{\sum_{j=1}^K \exp(z_j)},
\)
which is exactly the softmax mapping.
\end{proof}

Theorem~\ref{thm:softmax_simplified} is structural: it concerns only the
classifier parametrization and makes no assumption on the training loss.
Thus, whenever optimization depends only on the predicted
probabilities, standard softmax networks belong naturally to the simplified evidential framework.

\subsection{Plug-in approximation for smooth losses}

We can now derive bounds showing that the discrepancy between
\(\ell_{\mathrm{EDL}}(\valpha,\vy)\) and the plug-in loss
\(\ell_{\mathrm{plug}}(\valpha,\vy) \) is controlled by the concentration parameter
\(\alphazero\). This provides theoretical support for simplified
evidential objectives depending only on projected probabilities. 
We consider sufficiently smooth losses that are twice continuously
differentiable on the relevant region of the simplex, and we obtain the plug-in approximation
by a first-order expansion around the Dirichlet mean.

\begin{theorem}[Explicit first-order expansion]
\label{prop:explicit_plugin}
Assume that for each label \(\vy\), the loss
\(\ell(\cdot,\vy):\Delta^{K-1}\to\R\) is twice continuously
differentiable on a neighborhood of the relevant region of the simplex,
with bounded and locally Lipschitz Hessian. Let
\(\vpi\sim\Dir(\valpha)\), and let $\hat{\vp}=\Piop(\valpha)$. Then
\begin{gather}
\ell_{\mathrm{EDL}}(\valpha,\vy)
=
\ell_{\mathrm{plug}}(\valpha,\vy)
+
R(\valpha,\vy),
\end{gather}
where the remainder satisfies
$R(\valpha,\vy)=O\!\left((\alphazero+1)^{-1}\right).$

\end{theorem}
\begin{proof}[Proof Sketch]
We first apply the second-order Taylor bound and then apply Theorem \ref{thm:taylor_C22M} to bound the remainder term. After applying the definition of the covariance for Dirichlet random variables, one obtains that the remainder is of size $O\!\left((\alphazero+1)^{-1}\right)$.
Further, one can see that the first-order term vanishes in expectation and the second-order term is also of size $O\!\left((\alphazero+1)^{-1}\right)$ by restructuring it and using the definition of the covariance for Dirichlet random variables. Proof details
and explicit non-asymptotic bounds are provided in
Appendix~\ref{app:dirichlet_expansion}.
\end{proof}
The preceding results show that Dirichlet-expected evidential losses can
be replaced, up to concentration-controlled correction terms, by plug-in
losses evaluated at projected probabilities:
\begin{gather}
\ell_{\mathrm{EDL}}(\valpha,\vy)
\approx
\ell_{\mathrm{plug}}(\valpha,\vy).
\end{gather}
This yields simplified evidential objectives that depend only on
\(\Piop(\valpha)\), eliminating the need to explicitly evaluate
Dirichlet-expected losses while retaining their dominant first-order
behavior.

\begin{remark}
    For losses that are only Lipschitz continuous on the simplex, one can still
    approximate the Dirichlet-expected evidential loss by the plug-in loss, but
    with a weaker remainder \(O((\alpha_0+1)^{-1/2})\). See
    Appendix~\ref{app:proof_of_leme_lipschitz_plugin}.
\end{remark}

The Taylor-based approximation in Theorem~\ref{prop:explicit_plugin}
requires uniformly bounded curvature on the relevant region of the
simplex. This condition is natural for smooth losses such as the
mean-squared error loss, but it is not globally valid for cross-entropy:
for \(\ell(\vpi,y)=-\log \pi_y\), the derivatives become singular as
\(\pi_y\to 0\). Since cross-entropy is central in classification, we state
a separate result based on the Dirichlet logarithmic moment, under an
interiority condition on the projected class probability.

\begin{lemma}[Cross-entropy plug-in correction]
\label{lem:ce_plugin_correction-main}
Let
\(
\vpi\sim\Dir(\valpha)
\),
\(
\hat\vp=\Piop(\valpha)=\valpha/\alphazero
\), and
\(
\alphazero=\sum_{j=1}^K \alpha_j
\).
Let \(\vy\in\{0,1\}^K\) be a one-hot target vector, and let
\(y\in\{1,\dots,K\}\) denote its target class index, so that \(y_j=1\) if
and only if \(j=y\). Consider the cross-entropy loss
$
\ell(\vpi,\vy)
=
-\sum_{j=1}^K y_j \log \pi_j
=
-\log \pi_y .
$
Writing \(p_y=\alpha_y/\alphazero\), if \(p_y\ge\delta>0\), then
\begin{gather}
\EE_{\vpi\sim\Dir(\valpha)}
\big[\ell(\vpi,\vy)\big]
=
-\log p_y
+
O_{\delta}(\alphazero^{-1}).
\end{gather}
\end{lemma}
\begin{proof}[Proof Sketch]
Dirichlet logarithmic moment gives \(
\EE[-\log \pi_y]
=
\psi(\alphazero)-\psi(\alpha_y),
\)
where \(\psi\) denotes the digamma function. Rewriting
\(\alpha_y=p_y\alphazero\) separates the plug-in term \(-\log p_y\) from
two digamma correction terms. The standard asymptotic bound
\(|\psi(t)-\log t|=O(t^{-1})\) gives an error of order
\(O(\alphazero^{-1})+O(\alpha_y^{-1})\). Since \(p_y\ge\delta\), we have
\(\alpha_y\ge\delta\alphazero\), and hence the total correction is
\(O_\delta(\alphazero^{-1})\).
\end{proof}

\begin{remark}
Our bounds should be read as high-concentration approximations, corresponding 
to the high-evidence regime that evidential training is intended to 
reach for well-supported predictions. In this regime, Dirichlet-expected losses and their plug-in
counterparts have the same leading-order behavior. Since plug-in losses depend
on \(\valpha\) only through \(\Piop(\valpha)\), they isolate the
projected-probability component of the objective from the concentration scale.
The concentration scale remains a separate modeling degree of freedom;
Section~\ref{sec:experiments} empirically shows, through matched no-KL and
KL-regularized variants, that explicit regularization can substantially
change concentration-based selective prediction.
\end{remark}

\subsection{Interpretation of the KL regularizer}
\label{app:kl_interpretation}

Classical EDL objectives often include a KL-divergence regularizer between a
Dirichlet distribution induced by the network output and a non-informative
Dirichlet prior. From the first-order ERM perspective adopted in this work,
this term is interpreted as a deterministic regularization function of the
evidential parameters \(\valpha\). The Dirichlet distribution provides an
analytic form for the penalty, but the resulting objective remains an ordinary
deterministic training objective optimized by standard gradient-based methods.

In the classical construction, the KL term is applied after replacing the
target-class component by its prior value. Operationally, this removes the
target-class evidence from the regularizer, so the penalty mainly discourages
evidence assigned to incorrect classes. The KL term is therefore not required
to define the simplified evidential classifier or the plug-in objectives, but
it can affect the learned evidence scale and the induced class probabilities.
Consequently, it may influence both concentration-based uncertainty scores,
such as vacuity, and probability-based scores, such as predictive entropy. This
motivates the KL and no-KL variants included in the experimental comparison.
\section{Experiments}
\label{sec:experiments}

We evaluate classical evidential classifiers
(Definition~\ref{def:classical-evidential-classifier}) and simplified
evidential classifiers
(Definition~\ref{def:simplified-evidential-classifier}) using both
Dirichlet-expected losses and their plug-in counterparts
(Definition~\ref{def:plugin-evidential-loss}). Our goal is to test whether
the simplified framework preserves predictive accuracy and the operational
usefulness of uncertainty for selective prediction. Motivated by
Theorem~\ref{thm:softmax_simplified}, we also include softmax-style
parameterizations and evaluate all variants under a common
selective-prediction protocol.

\subsection{Experimental setup}
\label{sec:experimental_setup}

We evaluate all models on the Google Speech Commands v1 dataset
(\cite{warden2018speechcommands}) using the full \(30\)-class classification
task and the official train, validation, and test splits. All experiments use
the same MatchboxNet backbone (\cite{majumdar2020matchboxnet}) and follow the
NVIDIA NeMo MatchboxNet preprocessing and augmentation pipeline
(\cite{kuchaiev2019nemo}). They also use the same base optimizer,
learning-rate schedule, and number of training epochs, except for the
objective-specific regularization choices summarized in
Table~\ref{tab:model_variants}. Thus, the comparisons isolate the effect of
the evidential parameterization and training objective rather than differences
in architecture, data processing, or optimization setup. Full configuration
details are provided in Appendix~\ref{app:experimental_details}.

\begin{table}[!htbp]
\centering
\caption{Model variants used in the experiments. Here
\(e_i=\tau(z_i)\), \(\valpha=\phi(\ve)\), and
\(\hat{\vp}=\Piop(\valpha)\). For KL-regularized models,
\(\lambda_t=\min(1,t/T)\), where \(t\) denotes the epoch, and the KL term is
computed using shifted Dirichlet parameters
\(\alpha_i^{\mathrm{KL}}=e_i+1\).}
\label{tab:model_variants}
{\scriptsize
\setlength{\tabcolsep}{2.0pt}
\renewcommand{\arraystretch}{1.08}
\begin{tabular}{@{}lccccccccc@{}}
\toprule
&
EDL-CE
& EDL-CE no KL
& EDL-MSE
& Plug-in CE
& Plug-in MSE
& Softmax
& Softplus
& Softmax+KL
& Softmax+EDL-CE
\\
\midrule
\(\tau(z_i)\)
& softplus
& softplus
& softplus
& softplus
& softplus
& \(\exp\)
& softplus
& \(\exp\)
& \(\exp\)
\\
\(\phi(e_i)\)
& \(e_i+1\)
& \(e_i+1\)
& \(e_i+1\)
& \(e_i+1\)
& \(e_i+1\)
& \(e_i\)
& \(e_i\)
& \(e_i\)
& \(e_i\)
\\
Loss
& Dir. CE+KL
& Dir. CE
& Dir. MSE+KL
& CE\((\hat{\vp})\)
& MSE\((\hat{\vp})\)
& CE\((\hat{\vp})\)
& CE\((\hat{\vp})\)
& CE\((\hat{\vp})\) + KL
& Dir. CE
\\
\(T\)
& 400
& --
& 600
& --
& --
& --
& --
& 400
& --
\\
\bottomrule
\end{tabular}
}
\end{table}
The Softmax model recovers the standard softmax classifier by
Theorem~\ref{thm:softmax_simplified}. The Softplus model is the analogous \(c=0\) projected classifier using a softplus evidence map instead of an exponential.

\subsection{Selective-prediction protocol}
\label{sec:selective_prediction_protocol}

We evaluate selective prediction using two uncertainty scores:
\begin{gather}
u_{\mathrm{vacuity}}=\frac{K}{\alpha_0},
\qquad
u_{\mathrm{entropy}}
=
\frac{-\sum_{j=1}^K \hat p_j \log \hat p_j}{\log K}
\end{gather}
where \(\alpha_0=\sum_{j=1}^K \alpha_j\). All scores are computed for
every model under the common prediction parametrization introduced above and
evaluated using the same thresholding protocol. For models with \(c=0\), we
compute vacuity using the shifted concentration
\(\alpha_i^s=e_i+1\), so that \(\alpha_0^s=\sum_{j=1}^K e_j+K\); this preserves the uncertainty
ordering.

Given an uncertainty score and threshold \(t\), predictions with uncertainty
above \(t\) are withheld. With \(n_c\), \(n_f\), and \(n_w\) denoting
correct, incorrect, and withheld samples, we report
\(\mathrm{Acc}_{\mathrm{th}}=\frac{n_c}{n_c+n_f}\),
\(\mathrm{Acc}_{\mathrm{total}}=\frac{n_c}{n_c+n_f+n_w}\), and
\(\mathrm{Coverage}=\frac{n_c+n_f}{n_c+n_f+n_w}\). The main tables report \(\mathrm{Acc}_{\mathrm{total}}\) at operating points selected by target \(\mathrm{Acc}_{\mathrm{th}}\) values of \(99.0\%\),
\(99.5\%\), and \(99.9\%\). Details of the operating-point selection rule are
given in Appendix~\ref{app:selective_prediction_details}.

\subsection{Main results}
\label{sec:main_results}

Table~\ref{tab:entropy_vacuity_total_accuracy_thresholds}
reports total accuracy at selective-prediction operating points selected to
match target thresholded accuracies, using entropy and vacuity as uncertainty
scores. Under entropy-based selection, the \(c=0\) simplified evidential
classifiers trained with CE\((\hat{\vp})\), namely Softmax and Softplus
models, achieve the best or near-best total accuracy across most operating
points. For example, at the \(99.9\%\) target, Softmax and Softplus models
achieve total accuracies of \(88.41\%\) and \(87.64\%\), respectively. This
shows that entropy remains a strong uncertainty score for selective
prediction on GSC V1, and that the simplified evidential classifiers trained with plug-in losses provide a strong baseline for both prediction and selective prediction.

The plug-in variants closely track their classical EDL counterparts across
both CE and MSE losses. Plug-in EDL-CE achieves nearly the same base accuracy
and total accuracy at the selective operating points as EDL-CE
under both entropy- and vacuity-based thresholding. The same pattern is
observed for EDL-MSE and Plug-in EDL-MSE. These results are consistent with
the approximation results in Section~\ref{sec:simplified_edl}: replacing the
Dirichlet-expected loss by the corresponding loss evaluated at the projected
Dirichlet mean preserves the main predictive and operational uncertainty
behavior on this benchmark.

The vacuity-based results highlight that Dirichlet-expected losses alone
do not make vacuity competitive with entropy. Without KL regularization, 
vacuity still contains some reliability information, since thresholding by vacuity improves thresholded accuracy. However, the no-KL variants degrade strongly at stricter
operating points. In particular, Softmax outperforms both EDL-CE no KL and
Softmax+EDL-CE under vacuity-based selection: at the \(99.9\%\) target,
Softmax obtains \(62.00\%\) total accuracy, compared with \(47.14\%\) for
EDL-CE no KL and \(50.81\%\) for Softmax+EDL-CE. Thus, replacing the standard
CE\((\hat{\vp})\) objective by the Dirichlet-expected CE loss does not by
itself make vacuity more effective. In contrast, adding KL regularization to
the Softmax model raises the corresponding vacuity-based total accuracy to
\(80.36\%\), while entropy-based selection remains broadly similar despite visible changes in the KDE plots in Figure~\ref{figure:kde_kl}. This
suggests that the KL regularizer substantially improves the reliability
ranking induced by vacuity.

Overall, the experiments indicate that the proposed simplified EDL framework,
based on plug-in losses and simplified evidential classifiers, captures the
main predictive and selective-prediction behavior of the classical EDL
framework on this benchmark. The results also show that KL regularization
mainly affects the usefulness of vacuity, rather than entropy, for selective
prediction.

It should be noted that in some cases the two uncertainty metrics lead to
almost exactly the same coverage because the accepted sets selected by entropy
and vacuity largely overlap. In other cases, the rankings induced by the two
metrics differ substantially, leading to different selective-prediction
behavior.

\begin{table}[!htbp]
\centering
\caption{
Total accuracy under entropy- and vacuity-based selective prediction on the
GSC V1 test split. Values are means over 5 runs with
{\scriptsize$\pm 2\sigma$}. Columns report \(\mathrm{Acc}_{\mathrm{total}}\)
at the largest-coverage operating point satisfying the indicated
\(\mathrm{Acc}_{\mathrm{th}}\) target up to numerical tolerance.
Corresponding entropy-based curves are shown in
Figure~\ref{fig:appendix_entropy_threshold_curves}.
}
\label{tab:entropy_vacuity_total_accuracy_thresholds}
{\small
\setlength{\tabcolsep}{2.4pt}
\begin{tabular}{@{}lcccc@{\hspace{1.0em}}ccc@{}}
\toprule
&
&
\multicolumn{3}{c}{Entropy}
&
\multicolumn{3}{c}{Vacuity}
\\
\cmidrule(lr){3-5}
\cmidrule(lr){6-8}
Model
& Base Acc.
& 99.0\% & 99.5\% & 99.9\%
& 99.0\% & 99.5\% & 99.9\%
\\
\midrule
Softmax
& \textbf{97.21}{\tiny$\pm$0.18}
& \textbf{96.47}{\tiny$\pm$0.59}
& 94.50{\tiny$\pm$1.23}
& \textbf{88.41}{\tiny$\pm$2.39}
& 95.54{\tiny$\pm$1.38}
& 86.57{\tiny$\pm$2.93}
& 62.00{\tiny$\pm$17.09}
\\
Softplus
& 97.07{\tiny$\pm$0.16}
& 96.14{\tiny$\pm$0.46}
& \textbf{94.81}{\tiny$\pm$0.63}
& 87.64{\tiny$\pm$6.41}
& 63.93{\tiny$\pm$6.19}
& 58.75{\tiny$\pm$4.13}
& 52.87{\tiny$\pm$6.31}
\\
Softmax + KL
& 96.84{\tiny$\pm$0.18}
& 95.79{\tiny$\pm$0.35}
& 94.19{\tiny$\pm$0.80}
& 84.53{\tiny$\pm$3.12}
& 95.71{\tiny$\pm$0.35}
& \textbf{94.13}{\tiny$\pm$0.75}
& 80.36{\tiny$\pm$7.60}
\\
Softmax + EDL-CE
& 97.03{\tiny$\pm$0.32}
& 96.12{\tiny$\pm$0.71}
& 94.79{\tiny$\pm$0.70}
& 87.76{\tiny$\pm$6.57}
& 94.45{\tiny$\pm$2.24}
& 84.79{\tiny$\pm$3.67}
& 50.81{\tiny$\pm$24.50}
\\
EDL-CE
& 96.88{\tiny$\pm$0.37}
& 95.76{\tiny$\pm$0.75}
& 93.61{\tiny$\pm$1.26}
& 81.61{\tiny$\pm$5.80}
& \textbf{95.73}{\tiny$\pm$0.78}
& 93.55{\tiny$\pm$1.23}
& 81.62{\tiny$\pm$5.81}
\\
EDL-CE no KL
& 96.68{\tiny$\pm$0.26}
& 94.81{\tiny$\pm$1.17}
& 92.40{\tiny$\pm$1.26}
& 75.58{\tiny$\pm$14.11}
& 89.28{\tiny$\pm$3.83}
& 81.09{\tiny$\pm$6.57}
& 47.14{\tiny$\pm$15.60}
\\
Plug-in EDL-CE
& 96.84{\tiny$\pm$0.38}
& 95.68{\tiny$\pm$0.67}
& 93.39{\tiny$\pm$1.57}
& 83.55{\tiny$\pm$6.58}
& 95.68{\tiny$\pm$0.68}
& 93.40{\tiny$\pm$1.56}
& \textbf{83.55}{\tiny$\pm$6.58}
\\
\addlinespace[0.15em]
EDL-MSE
& 96.55{\tiny$\pm$0.19}
& 94.91{\tiny$\pm$0.37}
& 92.87{\tiny$\pm$0.95}
& 80.93{\tiny$\pm$5.09}
& 94.89{\tiny$\pm$0.39}
& 92.87{\tiny$\pm$0.95}
& 80.93{\tiny$\pm$5.09}
\\
Plug-in EDL-MSE
& 96.55{\tiny$\pm$0.10}
& 94.88{\tiny$\pm$0.31}
& 92.46{\tiny$\pm$0.62}
& 82.53{\tiny$\pm$6.08}
& 94.86{\tiny$\pm$0.26}
& 92.46{\tiny$\pm$0.62}
& 82.53{\tiny$\pm$6.08}
\\
\bottomrule
\end{tabular}
}
\end{table}

\begin{figure*}[!htbp]
\centering
\begin{minipage}{0.25\linewidth}
    \centering
    \includegraphics[width=\linewidth]{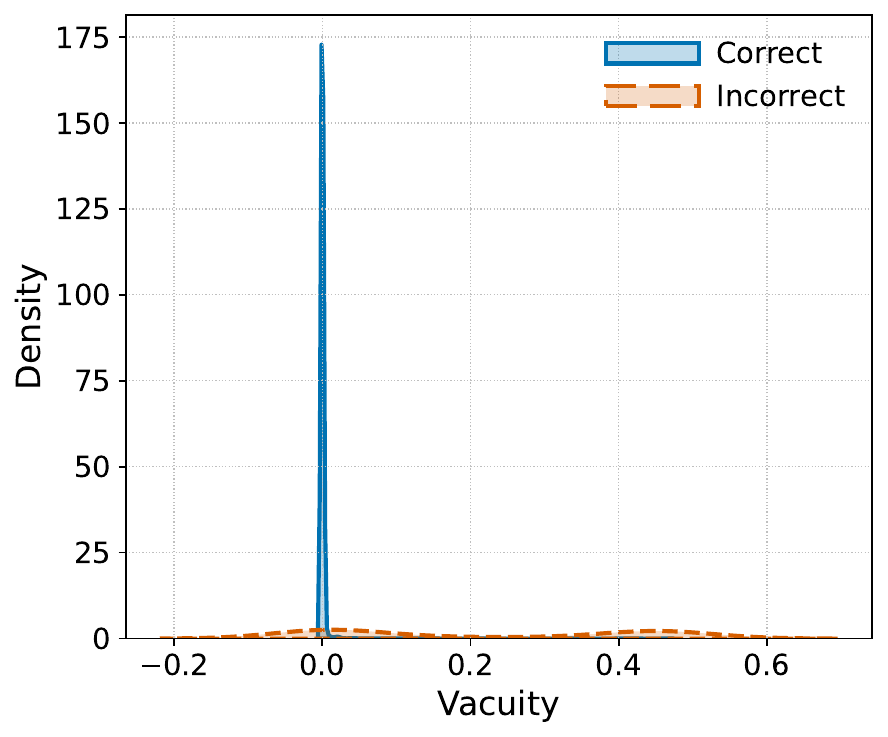}\\[-0.3em]
    {\scriptsize Softmax}
\end{minipage}\hfill
\begin{minipage}{0.25\linewidth}
    \centering
    \includegraphics[width=\linewidth]{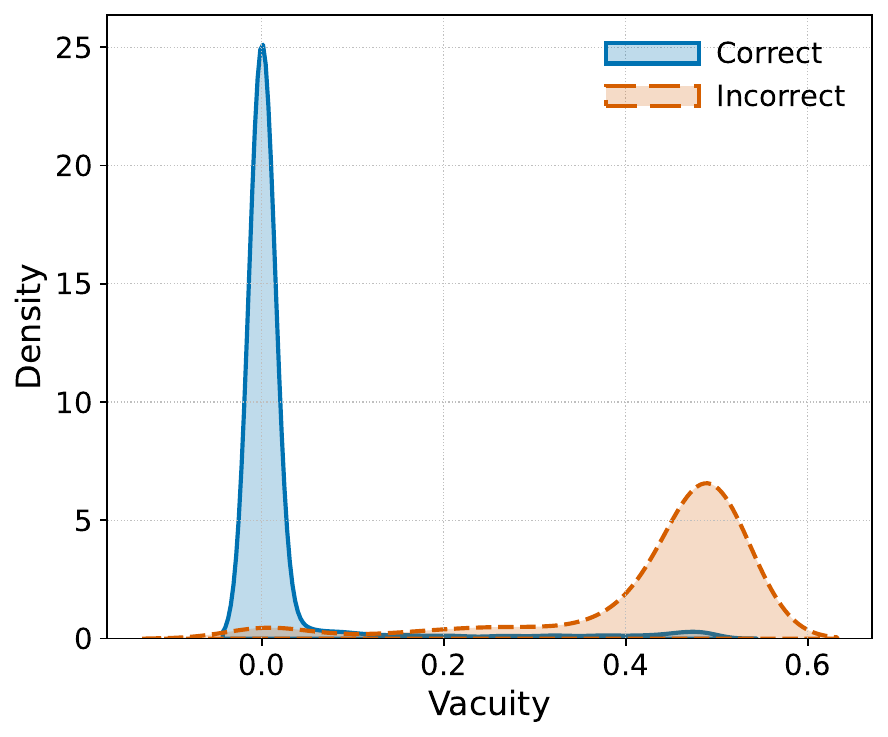}\\[-0.3em]
    {\scriptsize Softmax + KL}
\end{minipage}\hfill
\begin{minipage}{0.25\linewidth}
    \centering
    \includegraphics[width=\linewidth]{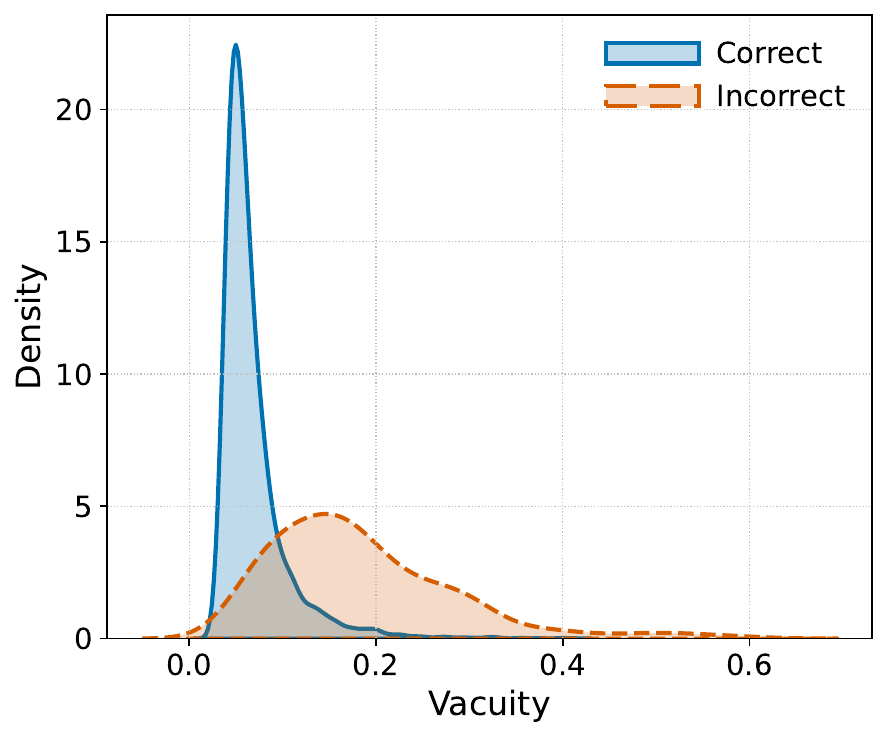}\\[-0.3em]
    {\scriptsize EDL-CE no KL}
\end{minipage}\hfill
\begin{minipage}{0.25\linewidth}
    \centering
    \includegraphics[width=\linewidth]{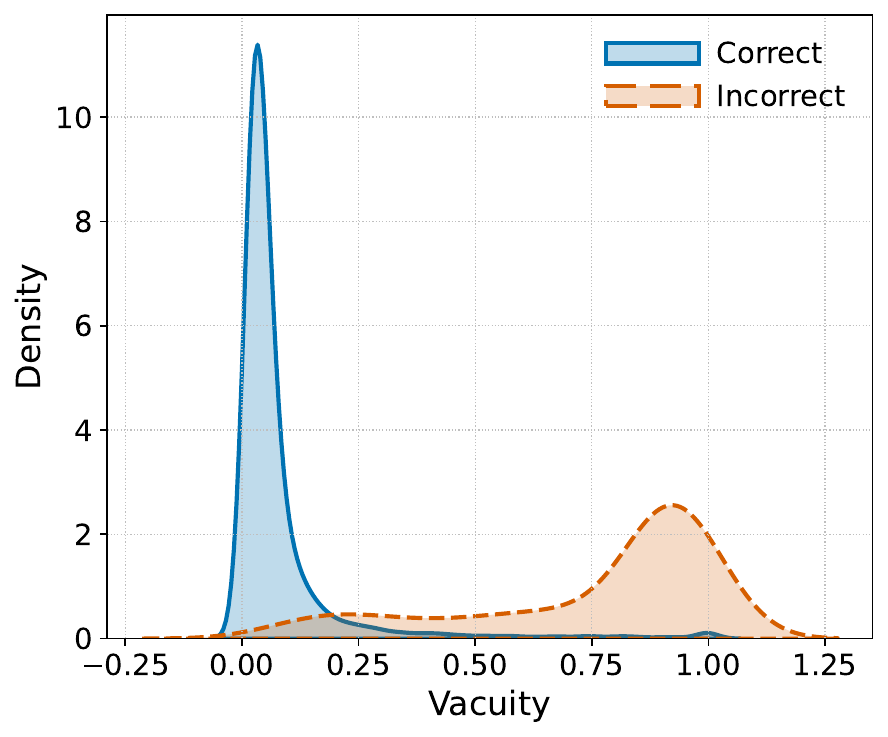}\\[-0.3em]
    {\scriptsize EDL-CE}
\end{minipage}
\caption{
Vacuity KDEs for correctly and incorrectly classified test samples on GSC V1.
The selected pairs compare KL-regularized models with their no-KL counterparts;
additional KDE plots are provided in Appendix~\ref{app:additional_kde}.
}
\label{figure:kde_kl}
\end{figure*}

\section{Conclusion and future directions}
\label{sec:conclusion}
In this work, we showed that classical EDL objectives can be approximated by a class of simplified plug-in objectives. Empirically we demonstrated that these approximations perform as well as their classical counterparts for predictive and selective prediction on the Google Speech Command v1 dataset. These results suggest that there is a strong practical value of evidential methods when output parametrization, plug-in losses, and potentially regularization are combined.

These findings motivate future work on the simplified EDL framework, consisting 
of plug-in losses (Definition~\ref{def:plugin-evidential-loss}) and simplified 
evidential classifiers (Definition~\ref{def:simplified-evidential-classifier}), 
together with explicit uncertainty-aware regularization. In particular, 
softmax-style models, viewed as simplified evidential classifiers by
Theorem~\ref{thm:softmax_simplified}, provide a natural starting point for
studying decision-relevant uncertainty while retaining simple single-pass
inference.

A limitation of the present study is that the empirical evaluation is restricted
to a single keyword-spotting benchmark and in-distribution selective prediction.
Future work should therefore evaluate the same framework on a wider range of
datasets, architectures, and distribution-shift settings. This is especially
important because selective prediction can improve deployment reliability by
identifying uncertain situations in which additional safety measures or human
review may be needed.

A natural follow-up question would be to understand when the uncertainty ordering induced by vacuity and entropy does or does not agree locally or globally, and how to best combine these two uncertainty metrics through a two-threshold selective-prediction rule.

\begin{ack}
FK gratefully acknowledges the support of the German Science Foundation (DFG) in the context of the priority program Theoretical Foundations of Deep Learning (project KR 4512/6-1). HL and FK gratefully acknowledge the support of the Munich Center for Machine Learning (MCML).

\end{ack}

\bibliographystyle{plainnat}
\bibliography{refs}

\appendix

\section{Explicit expansion of Dirichlet-expected losses}
\label{app:dirichlet_expansion}
This appendix provides an explicit asymptotic expansion of Dirichlet-expected
losses around the Dirichlet mean. The result characterizes how Dirichlet-based
objectives differ from plug-in losses evaluated at the projected probabilities
and shows that the additional contribution enters as a variance-induced
correction term arising from the Dirichlet covariance. This expansion clarifies the structural
role of projected probabilities in evidential classifiers and supports the
classifier definitions adopted in the main text.

\begin{theorem}[Second-order Taylor bound for twice
continuously differentiable Lipschitz functions {\cite[Lemma~1.2.4]{nesterov2013introductory}}] 
\label{thm:taylor_C22M}
Let $U\subset\mathbb{R}^K$ be open and let $f:U\to\mathbb{R}$ be twice
continuously differentiable. Assume that $\nabla^2 f$ is Lipschitz on $U$ with
constant $M\ge 0$ in operator norm, i.e.
\[
\|\nabla^2 f(u)-\nabla^2 f(v)\|_{\mathrm{op}}
\le
M\|u-v\|_2,
\qquad \forall\,u,v\in U.
\]
Then for all $x,y\in U$ such that the line segment $[x,y]\subset U$, we have
\begin{gather}
\label{eq:taylor_grad_remainder_bound}
\big\|\nabla f(y)-\nabla f(x)-\nabla^2 f(x)(y-x)\big\|_2
\le
\frac{M}{2}\,\|y-x\|_2^2,
\\[2mm]
\label{eq:taylor_func_remainder_bound}
\Big|
f(y)-f(x)-\nabla f(x)^\top (y-x)
-\frac12 (y-x)^\top \nabla^2 f(x)(y-x)
\Big|
\le
\frac{M}{6}\,\|y-x\|_2^3.
\end{gather}
\end{theorem}

Given the above we can restate our main result Theorem \ref{prop:explicit_plugin}.
\begin{theorem}[Explicit expansion of Dirichlet-expected losses]
\label{lem:dirichlet_erm_loss}
Let
\begin{gather}
\vpi \sim \Dir(\valpha),
\qquad
\valpha\in\R_{+}^{K},
\qquad
\alphazero := \sum_{i=1}^{K}\alpha_i,
\qquad
\hat \vp := \Piop(\valpha)
=
\frac{\valpha}{\alphazero}
\in\Delta^{K-1}.
\end{gather}
Fix a label \(\vy\). Assume that \(\ell(\cdot,\vy)\) is twice continuously
differentiable on \(\Delta^{K-1}\), with Hessian
\begin{gather}
H_{\ell}(q,y)
:=
\nabla_q^2\ell(q,y).
\end{gather}
Assume further that there exist constants \(M\ge0\) and \(G\ge0\) such that,
for all \(q,u,v\in\Delta^{K-1}\),
\begin{align}
\label{eq:hessian_op_bound}
\opnorm{H_{\ell}(q,y)}
&\le M,
\\
\label{eq:hessian_lipschitz}
\opnorm{H_{\ell}(u,y)-H_{\ell}(v,y)}
&\le
G\norm{u-v}_2 .
\end{align}
Then
\begin{align}
\ell_{\mathrm{EDL}}(\valpha,\vy) = \EE_{\vpi\sim\Dir(\valpha)}[\ell(\vpi,\vy)]
&=
\ell(\hat \vp,\vy)
+
\frac{1}{\alphazero+1}L_1(\hat \vp,\vy)
+
R(\valpha,\vy) \\
&=  \ell_{\mathrm{plug}}(\valpha,\vy) + O\!\left((\alphazero+1)^{-1}\right)
\end{align}
where
\begin{gather}
\label{eq:l1def}
L_1(\hat \vp,\vy)
:=
\frac12
\mathrm{tr}\!\Big(
H_{\ell}(\hat \vp,\vy)
\big(\mathrm{Diag}(\hat \vp)-\hat \vp\hat \vp^\top\big)
\Big).
\end{gather}

\end{theorem}

\begin{proof}
Fix \(\vy\) and set
\begin{gather}
h:= \vpi-\hat \vp.
\end{gather}
Since \(\vpi,\hat \vp\in\Delta^{K-1}\) and the simplex is convex, the line segment
between \(\hat \vp\) and \(\vpi\) is contained in \(\Delta^{K-1}\). Applying Taylor's
theorem with Lipschitz Hessian to \(q\mapsto\ell(q,\vy)\) along this segment gives
\begin{gather}
\ell(\vpi,\vy)
=
\ell(\hat \vp,\vy)
+
\nabla\ell(\hat \vp,\vy)^\top h
+
\frac12 h^\top H_{\ell}(\hat \vp,\vy)h
+
r(\vpi,\vy),
\end{gather}
Given Theorem \ref{thm:taylor_C22M} we obtain the following bound of the remainder term
\begin{gather}
|r(\vpi,\vy)|
\le
\frac{G}{6}\norm{h}_2^3
=
\frac{G}{6}\norm{\vpi-\hat \vp}_2^3 .
\end{gather}
Taking expectations and using \(\EE[\vpi-\hat \vp]=0\), we obtain
\begin{gather}
\EE[\ell(\vpi,\vy)]
=
\ell(\hat \vp,\vy)
+
\frac12
\EE\!\left[
h^\top H_{\ell}(\hat \vp,\vy)h
\right]
+
\EE[r(\vpi,\vy)].
\end{gather}
The quadratic term satisfies
\begin{gather}
\EE\!\left[
h^\top H_{\ell}(\hat \vp,\vy)h
\right]
=
\mathrm{tr}\!\left(
H_{\ell}(\hat \vp,\vy)\,\mathrm{Cov}(\vpi)
\right).
\end{gather}
For \(\vpi\sim\Dir(\valpha)\),
\begin{gather}
\mathrm{Cov}(\vpi)
=
\frac{1}{\alphazero+1}
\left(
\mathrm{Diag}(\hat \vp)-\hat \vp\hat \vp^\top
\right).
\end{gather}
Therefore,
\begin{gather}
\EE[\ell(\vpi,\vy)]
=
\ell(\hat \vp,\vy)
+
\frac{1}{\alphazero+1}L_1(\hat \vp,\vy)
+
R(\valpha,\vy),
\end{gather}
here \(R(\valpha,\vy)
:=
\EE[r(\vpi,\vy)] \) and \(L_1\) as defined in Eq. \ref{eq:l1def}.

To bound \(L_1(\hat \vp,\vy)\), note that
\(\mathrm{Diag}(\hat \vp)-\hat \vp\hat \vp^\top\) is symmetric positive semidefinite. By
the trace bound,
\begin{align}
|L_1(\hat \vp,\vy)|
& \le
\frac12
\opnorm{H_{\ell}(\hat \vp,\vy)}
\,
\mathrm{tr}\!\left(
\mathrm{Diag}(\hat \vp)-\hat \vp\hat \vp^\top
\right)
\\
& =
\frac12
\opnorm{H_{\ell}(\hat \vp,\vy)}
(1-\norm{\hat \vp}_2^2)
\\
& \le
\frac12 M(1-\norm{\hat \vp}_2^2).
\end{align}
Since \(\norm{\hat \vp}_2^2\ge 1/K\) for every \(\hat \vp\in\Delta^{K-1}\), it follows that
\begin{gather}
|L_1(\hat \vp,\vy)|
\le
\frac12\left(1-\frac1K\right)M.
\end{gather}

We still need to bound the remainder by the Taylor remainder bound,
\begin{gather}
|R(\valpha,y)|
=
|\EE[r(\vpi,\vy)]|
\le
\EE|r(\vpi,\vy)|
\le
\frac{G}{6}\EE\norm{\vpi-\hat \vp}_2^3.
\end{gather}
Since \(\vpi,\hat \vp\in\Delta^{K-1}\), we have
\begin{gather}
\norm{\vpi-\hat \vp}_2\le 2,
\qquad
\norm{\vpi-\hat \vp}_2^3
\le
2\norm{\vpi-\hat \vp}_2^2.
\end{gather}
Consequently,
\begin{gather}
\EE\norm{\vpi-\hat \vp}_2^3
\le
2\EE\norm{\vpi-\hat \vp}_2^2
=
2\,\mathrm{tr}(\mathrm{Cov}(\vpi))
=
2\frac{1-\norm{\hat \vp}_2^2}{\alphazero+1}.
\end{gather}
Therefore,
\begin{gather}
|R(\valpha,\vy)|
\le
\frac{G}{6}\cdot
2\frac{1-\norm{\hat \vp}_2^2}{\alphazero+1}
=
\frac{G}{3}
\frac{1-\norm{\hat \vp}_2^2}{\alphazero+1}
\le
\frac{G}{3(\alphazero+1)}.
\end{gather}
Thus,
\begin{gather}
R(\valpha,\vy)
=
O\!\left((\alphazero+1)^{-1}\right)
\end{gather} which proves the theorem.
\end{proof}

Next, we will restate the cross-entropy specific plug-in correction (Lemma \ref{lem:ce_plugin_correction-main}) and prove it.
\begin{lemma}[Cross-entropy plug-in correction]
\label{lem:ce_plugin_correction}
Let
\begin{gather}
\vpi\sim\Dir(\valpha),
\qquad
\valpha\in\R_{+}^{K},
\qquad
\alphazero=\sum_{i=1}^{K}\alpha_i,
\qquad
\hat \vp=\Piop(\valpha)
=
\frac{\valpha}{\alphazero}.
\end{gather}
Let \(\vy\in\{0,1\}^K\) be a one-hot target vector, and let
\(y\in\{1,\dots,K\}\) denote its target class index, so that \(y_j=1\) if
and only if \(j=y\). Equivalently,
\begin{gather}
p_y=\frac{\alpha_y}{\alphazero}.
\end{gather}
For the one-hot cross-entropy loss
\( \ell(\vpi,\vy)=-\log \pi_y \), if \(p_y\ge\delta>0\),
we have
\begin{gather}
\EE[-\log \pi_y]
=
-\log p_y
+
\Delta_{\mathrm{CE}}(\valpha,\vy) \le  -\log p_y
+
O_{\delta}(\alphazero^{-1}),
\end{gather}
where
\begin{gather}
\Delta_{\mathrm{CE}}(\valpha,\vy)
:=
\big(\psi(\alphazero)-\log\alphazero\big)
-
\big(\psi(\alpha_y)-\log\alpha_y\big),
\end{gather}
and \(\psi\) denotes the digamma function.

\end{lemma}

\begin{proof}
For a Dirichlet random vector, the standard log-moment identity gives
\begin{gather}
\EE[\log \pi_y]
=
\psi(\alpha_y)-\psi(\alphazero).
\end{gather}
Therefore,
\begin{gather}
\EE[-\log \pi_y]
=
\psi(\alphazero)-\psi(\alpha_y).
\end{gather}
Since $p_y=\frac{\alpha_y}{\alphazero}$
the plug-in cross-entropy loss is
\begin{gather}
-\log p_y
=
-\log\left(\frac{\alpha_y}{\alphazero}\right)
=
\log\alphazero-\log\alpha_y.
\end{gather}
Subtracting the plug-in loss from the Dirichlet-expected loss gives
\begin{align}
\EE[-\log \pi_y]
-
(-\log p_y)
&=
\psi(\alphazero)-\psi(\alpha_y)
-
\log\alphazero
+
\log\alpha_y
\\
&=
\big(\psi(\alphazero)-\log\alphazero\big)
-
\big(\psi(\alpha_y)-\log\alpha_y\big)
=
\Delta_{\mathrm{CE}}(\valpha,\vy).
\end{align}
This proves the exact decomposition.

Using the standard bound
\begin{gather}
|\psi(t)-\log t|
\le
\frac{1}{t},
\qquad
t>0,
\end{gather}
we obtain
\begin{align}
|\Delta_{\mathrm{CE}}(\valpha,\vy)|
&\le
|\psi(\alphazero)-\log\alphazero|
+
|\psi(\alpha_y)-\log\alpha_y|
\\
&\le
\frac{1}{\alphazero}
+
\frac{1}{\alpha_y}.
\end{align}
Since
\( \alpha_y=\alphazero p_y, \)
this can be written as
\begin{gather}
|\Delta_{\mathrm{CE}}(\valpha,\vy)|
\le
\frac{1}{\alphazero}
+
\frac{1}{\alphazero p_y}.
\end{gather}
If \(p_y\ge\delta>0\), then
\begin{gather}
|\Delta_{\mathrm{CE}}(\valpha,\vy)|
\le
\frac{1}{\alphazero}
+
\frac{1}{\alphazero\delta}
=
\frac{1+\delta^{-1}}{\alphazero}.
\end{gather}
Thus,
\begin{gather}
\EE[-\log \pi_y]
=
-\log p_y
+
O_{\delta}(\alphazero^{-1}).
\end{gather}
\end{proof}

\section{Lipschitz plug-in approximation}
\label{app:proof_of_leme_lipschitz_plugin}
In this section, we will state the plug-in approximation for the class of Lipschitz functions, which in contrast to Theorem \ref{prop:explicit_plugin} does not require the loss to be differentiable.
\begin{lemma}[Lipschitz plug-in approximation]
\label{lem:lipschitz_plugin_edl}
Assume that for each label \(\vy\), the loss
\(\ell(\cdot,\vy):\Delta^{K-1}\to\mathbb{R}\) is \(L\)-Lipschitz with respect
to the Euclidean norm. Let
\(\boldsymbol{\pi}\sim\mathrm{Dir}(\boldsymbol{\alpha})\) and denote
\(\boldsymbol{p}=\Pi(\boldsymbol{\alpha})\). Then
\begin{equation}
\label{eq:lipschitz_plugin_edl}
\big|
\ell_{\mathrm{EDL}}(\valpha,\vy)
-
\ell_{\mathrm{plug}}(\valpha,\vy)
\big|
\le
L\sqrt{\frac{1-\|\boldsymbol{p}\|_2^2}{\alpha_0+1}}
\le
\frac{L}{\sqrt{\alpha_0+1}}.
\end{equation}
\end{lemma}

\begin{proof}
By the Lipschitz property of \(\ell\),
\begin{gather}
\big|\ell(\boldsymbol{\pi},\vy)-\ell(\boldsymbol{p},\vy)\big|
\le
L\|\boldsymbol{\pi}-\boldsymbol{p}\|_2 .
\end{gather}
Taking expectations and applying Jensen's inequality,
\begin{gather}
\big|
\ell_{\mathrm{EDL}}(\valpha,\vy)
-
\ell_{\mathrm{plug}}(\valpha,\vy)
\big| = \big|
\mathbb{E}_{\vpi\sim\Dir(\valpha)}[\ell(\boldsymbol{\pi},\vy)]-\ell(\boldsymbol{p},\vy)
\big|
\le
L\,\mathbb{E}\|\boldsymbol{\pi}-\boldsymbol{p}\|_2
\le
L\sqrt{\mathbb{E}\|\boldsymbol{\pi}-\boldsymbol{p}\|_2^2}.
\end{gather}
Since \(\boldsymbol{p}=\mathbb{E}[\boldsymbol{\pi}]\),
\begin{gather}
\mathbb{E}\|\boldsymbol{\pi}-\boldsymbol{p}\|_2^2
=
\mathrm{tr}(\mathrm{Cov}(\boldsymbol{\pi})).
\end{gather}
For a Dirichlet random vector,
\begin{gather}
\mathrm{Var}(\pi_k)=\frac{p_k(1-p_k)}{\alpha_0+1},
\end{gather}
hence
\begin{gather}
\mathrm{tr}(\mathrm{Cov}(\boldsymbol{\pi}))
=
\sum_{k=1}^K \frac{p_k(1-p_k)}{\alpha_0+1}
=
\frac{1-\|\boldsymbol{p}\|_2^2}{\alpha_0+1}.
\end{gather}
Substituting into the previous bound gives
\eqref{eq:lipschitz_plugin_edl}. The final inequality follows from
\(\|\boldsymbol{p}\|_2^2\ge 0\).
\end{proof}

\section{More experimental results and details}
\label{app:experimental_results}
\subsection{Experimental details}
\label{app:experimental_details}

We use the public Google Speech Commands v1 dataset (~\cite{warden2018speechcommands}),
released under the Creative Commons Attribution 4.0 license. Experiments are
implemented with NVIDIA NeMo (~\cite{kuchaiev2019nemo}), released under the
Apache 2.0 license, and use a MatchboxNet backbone following
\cite{majumdar2020matchboxnet}.

\paragraph{Hardware and software.}
Experiments are run on a workstation with an AMD Radeon RX 7900 XTX GPU and
an AMD Ryzen 7 7700 CPU under Windows 11 using WSL with Ubuntu 24.04. The
training environment uses ROCm 7.2.2 and NVIDIA NeMo for the MatchboxNet
implementation, preprocessing, and training pipeline.

\paragraph{Training setup.}
Following the MatchboxNet optimization setup
(\cite{majumdar2020matchboxnet}), all models are trained for \(200\) epochs
using NovoGrad with \(\beta_1=0.95\), \(\beta_2=0.5\), maximum learning rate
\(0.05\), and minimum learning rate \(10^{-3}\). We use a
warmup-hold-decay learning-rate schedule with \(5\%\) warmup, \(45\%\) hold,
and second-order polynomial decay for the remaining steps. Final runs use
batch size \(256\) on a single GPU. Weight decay is set to \(10^{-3}\) for models without KL regularization and to \(0\) for models trained with the KL regularizer, avoiding an additional
weight decay penalty on top of the explicit Dirichlet regularizer.

\paragraph{Preprocessing and augmentation.}
We follow the MatchboxNet preprocessing and augmentation pipeline
(\cite{majumdar2020matchboxnet}). Each waveform is converted into \(64\) MFCC
features computed from \(25\,\mathrm{ms}\) windows with \(10\,\mathrm{ms}\)
stride, and the temporal dimension is symmetrically zero-padded or cropped to
\(128\) feature vectors. During training, we apply random time shifts in
\([-5,5]\,\mathrm{ms}\), additive white noise with magnitude in
\([-90,-46]\,\mathrm{dB}\), SpecAugment with two time masks of width up to
\(25\) time steps and two frequency masks of width up to \(15\) frequency
bands, and SpecCutout with five rectangular masks using the same time and
frequency dimensions.

\paragraph{Runtime.}
A single \(200\)-epoch training run took approximately \(1\) hour on the
workstation with an AMD Radeon RX 7900 XTX GPU and an AMD Ryzen 7 7700 CPU.
The final experimental suite consisted of \(45\) independent training runs and
was completed in approximately \(22\)--\(24\) hours of wall-clock time by
running up to three jobs in parallel on the same GPU. Evaluation,
selective-prediction curve computation, and plotting were inexpensive compared
with training.

\subsection{Selective-prediction operating point selection}
\label{app:selective_prediction_details}

For each trained model and uncertainty score, we compute a
selective-prediction curve by sorting test samples in increasing uncertainty
and evaluating prefixes of this ordering. Each point on the curve corresponds
to a thresholded classifier that accepts samples below an uncertainty
threshold and withholds the remaining samples.

For a target thresholded accuracy \(a^\star\), we select the operating point
with maximum coverage among all thresholds satisfying
\[
\mathrm{Acc}_{\mathrm{th}} \ge a^\star - \varepsilon,
\qquad
\varepsilon=10^{-6}.
\]
The selected operating point is
\[
\tau^\star
\in
\arg\max_{\tau}
\mathrm{Coverage}(\tau)
\quad
\text{subject to}
\quad
\mathrm{Acc}_{\mathrm{th}}(\tau)
\ge
a^\star-\varepsilon .
\]
This asymmetric rule reflects the operational selective-prediction setting:
for a required reliability level, the goal is to retain as many samples as
possible. The tolerance \(\varepsilon\) is used only to avoid numerical
rounding effects.

For every selected operating point, we also verify the identity
\[
\mathrm{Acc}_{\mathrm{total}}
=
\mathrm{Acc}_{\mathrm{th}}\cdot \mathrm{Coverage},
\]
up to numerical precision.

\subsection{Accuracy vs coverage}

Figure~\ref{fig:appendix_entropy_threshold_curves} reports entropy-based
selective-prediction threshold curves for all model variants. For each run,
the normalized predictive-entropy threshold is varied, and the corresponding
coverage and thresholded accuracy are computed. Each panel therefore shows
directly how the operating point changes with the entropy threshold.
Stricter thresholds reject more samples, reducing coverage while typically
increasing the accuracy among the retained samples. The plots provide a
visual summary of this trade-off and of the variability across runs, whereas
the tables report the quantitative comparisons at selected operating points.
Solid lines denote the mean over five runs and shaded regions denote
\(\pm 2\sigma\).

\begin{figure}[!httbp]
\centering

\begin{minipage}{0.30\linewidth}
    \centering
    \includegraphics[width=\linewidth]{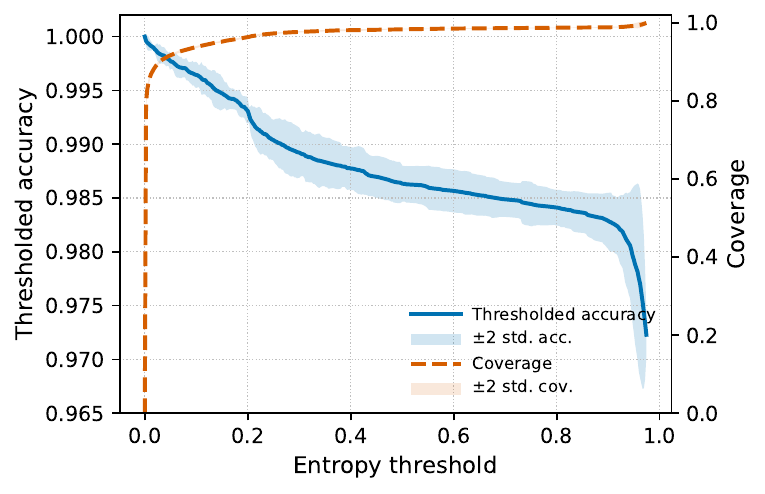}\\[-0.35em]
    {\scriptsize Softmax}
\end{minipage}
\hfill
\begin{minipage}{0.30\linewidth}
    \centering
    \includegraphics[width=\linewidth]{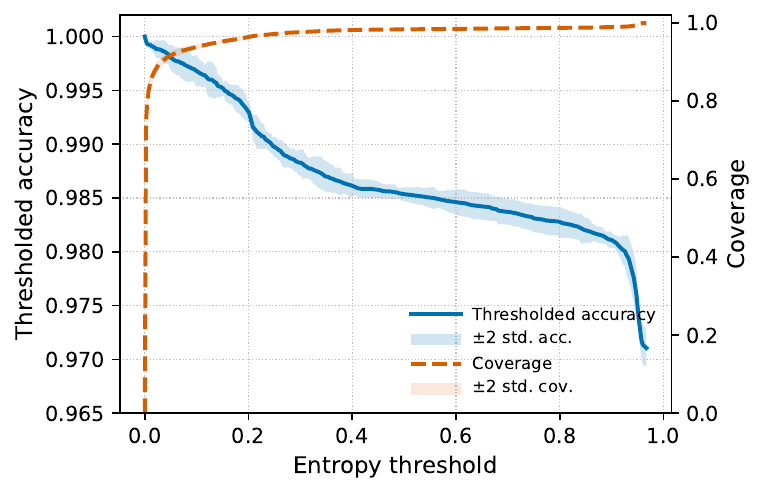}\\[-0.35em]
    {\scriptsize Softplus Proj.}
\end{minipage}
\hfill
\begin{minipage}{0.30\linewidth}
    \centering
    \includegraphics[width=\linewidth]{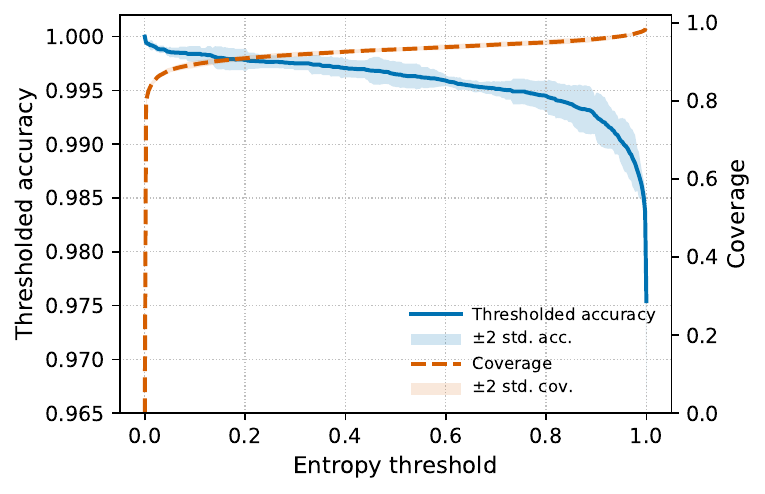}\\[-0.35em]
    {\scriptsize Softmax + KL}
\end{minipage}

\vspace{0.5em}

\begin{minipage}{0.30\linewidth}
    \centering
    \includegraphics[width=\linewidth]{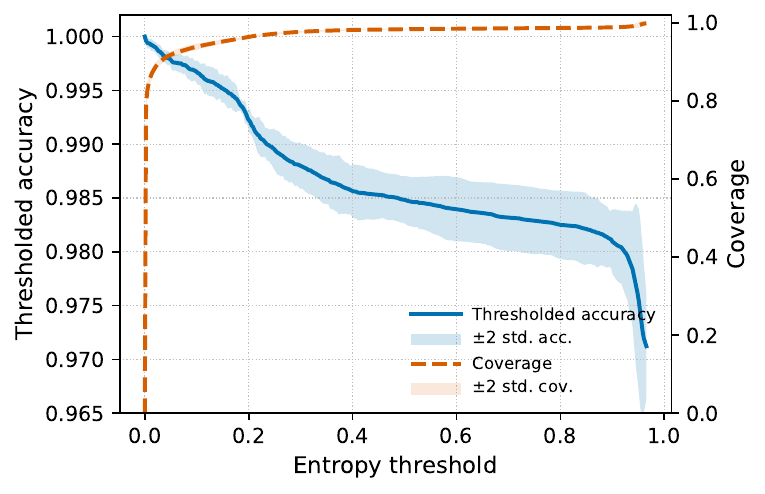}\\[-0.35em]
    {\scriptsize Softmax + EDL-CE}
\end{minipage}
\hfill
\begin{minipage}{0.30\linewidth}
    \centering
    \includegraphics[width=\linewidth]{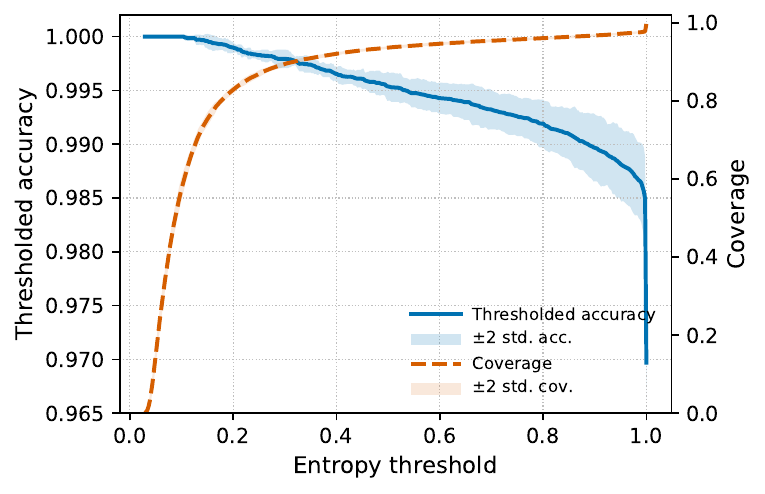}\\[-0.35em]
    {\scriptsize EDL-CE}
\end{minipage}
\hfill
\begin{minipage}{0.30\linewidth}
    \centering
    \includegraphics[width=\linewidth]{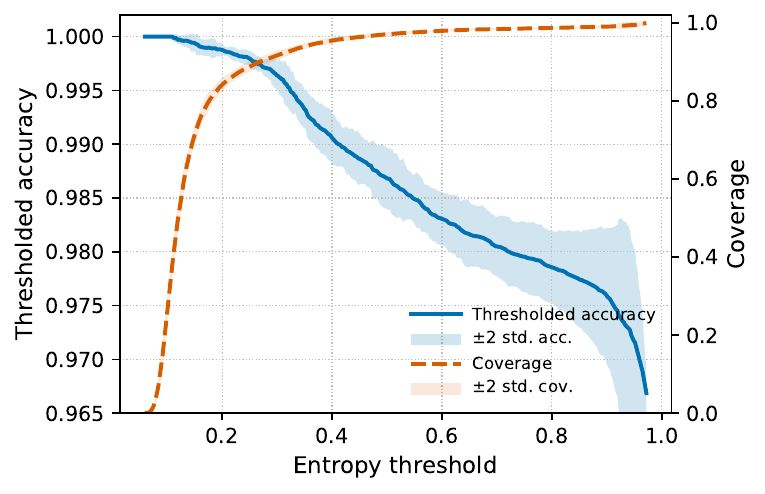}\\[-0.35em]
    {\scriptsize EDL-CE no KL}
\end{minipage}

\vspace{0.5em}

\begin{minipage}{0.30\linewidth}
    \centering
    \includegraphics[width=\linewidth]{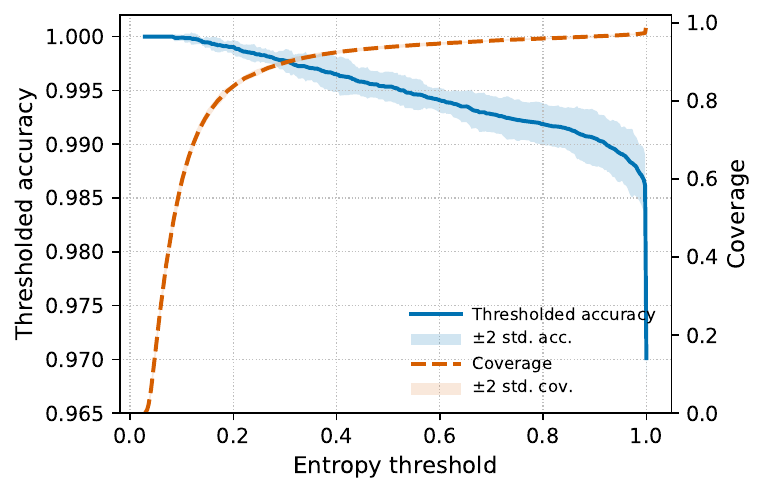}\\[-0.35em]
    {\scriptsize Plug-in EDL-CE}
\end{minipage}
\hfill
\begin{minipage}{0.30\linewidth}
    \centering
    \includegraphics[width=\linewidth]{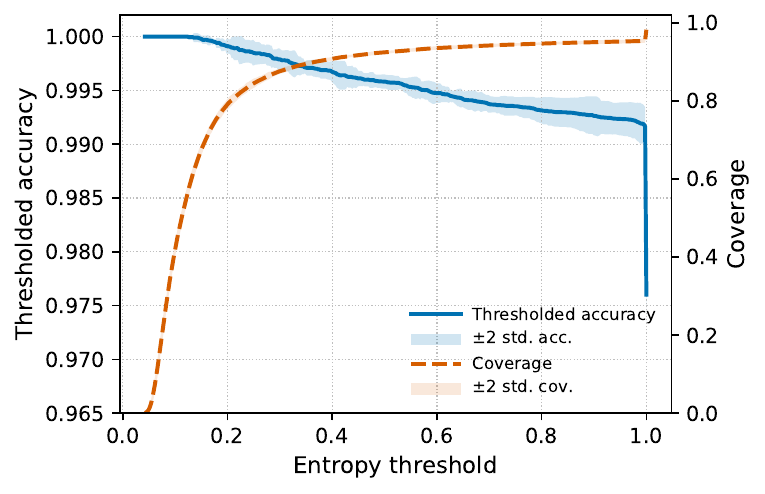}\\[-0.35em]
    {\scriptsize EDL-MSE}
\end{minipage}
\hfill
\begin{minipage}{0.30\linewidth}
    \centering
    \includegraphics[width=\linewidth]{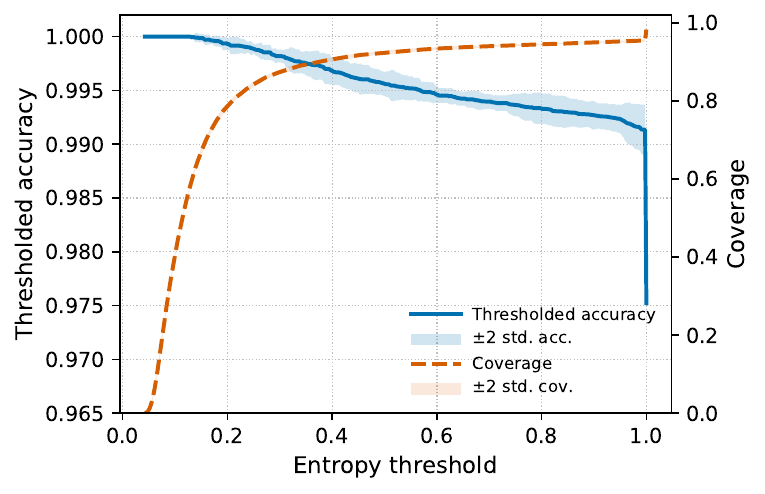}\\[-0.35em]
    {\scriptsize Plug-in EDL-MSE}
\end{minipage}

\caption{
Entropy-based selective-prediction threshold curves for all model variants on
GSC V1. Each panel shows thresholded accuracy and coverage as functions of the
normalized predictive-entropy threshold, averaged over five runs. The curves 
illustrate the selective-prediction trade-off: stricter uncertainty
thresholds reject more samples, reducing coverage while typically increasing
the accuracy among accepted samples. Solid lines show the mean
across runs and shaded regions denote \(\pm 2\sigma\). Quantitative
operating-point comparisons at target thresholded accuracies are reported in
the tables.
}
\label{fig:appendix_entropy_threshold_curves}
\end{figure}

\subsection{Additional uncertainty distributions}
\label{app:additional_kde}

Figures~\ref{fig:appendix_entropy_kde_all}
and~\ref{fig:appendix_vacuity_kde_remaining} show additional KDE plots of
the uncertainty scores for correctly and incorrectly classified test samples.
The entropy plots are shown for all model variants, while the vacuity plots
include the variants not shown in the main text. These figures are intended as
diagnostic visualizations of how each uncertainty score separates correct and
incorrect predictions; the selective-prediction tables provide the quantitative
comparison at fixed operating points.

\begin{figure}[!httbp]
\centering
\begin{minipage}{0.30\linewidth}
    \centering
    \includegraphics[width=\linewidth]{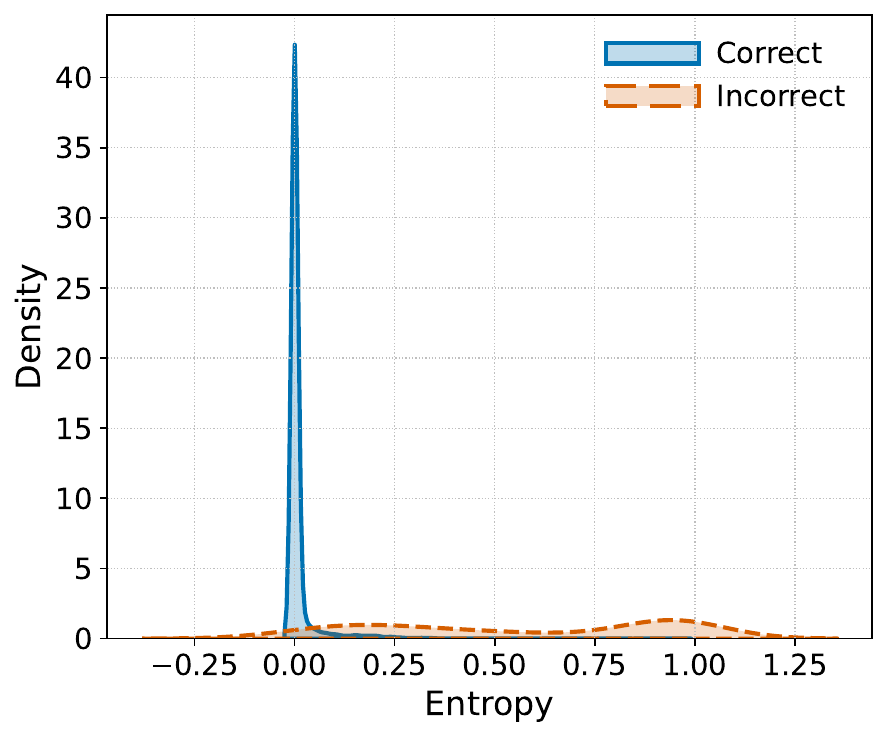}\\[-0.4em]
    {\scriptsize Softmax}
\end{minipage}
\hfill
\begin{minipage}{0.30\linewidth}
    \centering
    \includegraphics[width=\linewidth]{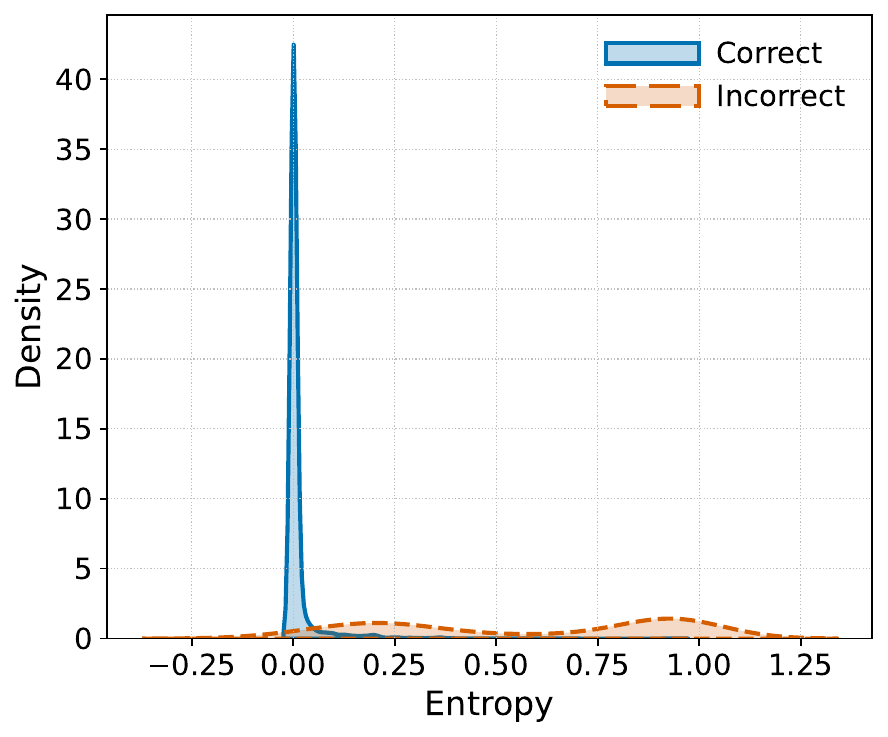}\\[-0.4em]
    {\scriptsize Softplus Proj.}
\end{minipage}
\hfill
\begin{minipage}{0.30\linewidth}
    \centering
    \includegraphics[width=\linewidth]{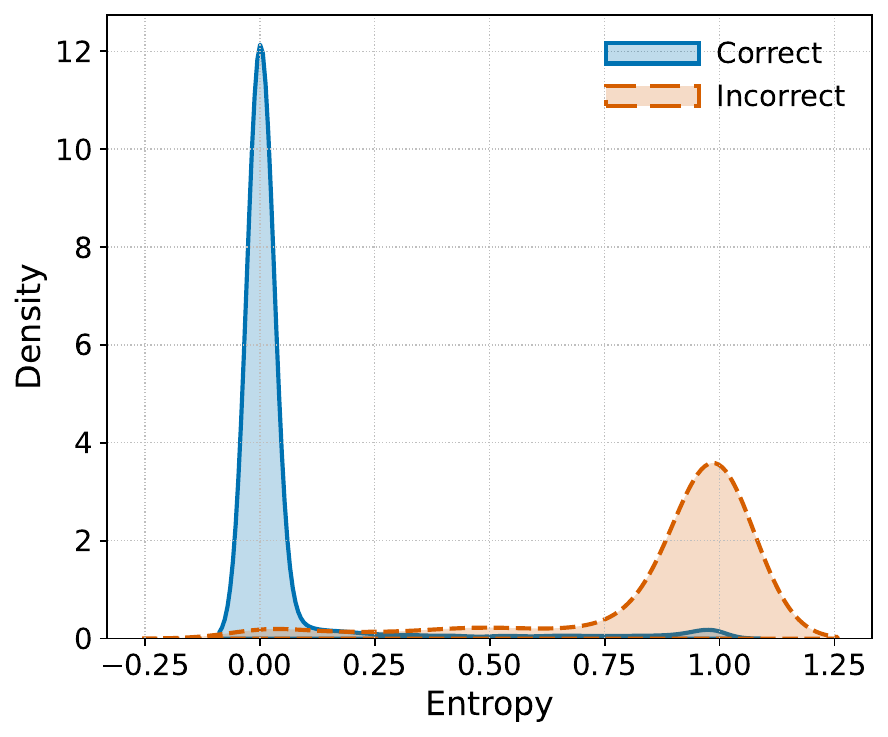}\\[-0.4em]
    {\scriptsize Softmax + KL}
\end{minipage}

\vspace{0.15em}

\begin{minipage}{0.30\linewidth}
    \centering
    \includegraphics[width=\linewidth]{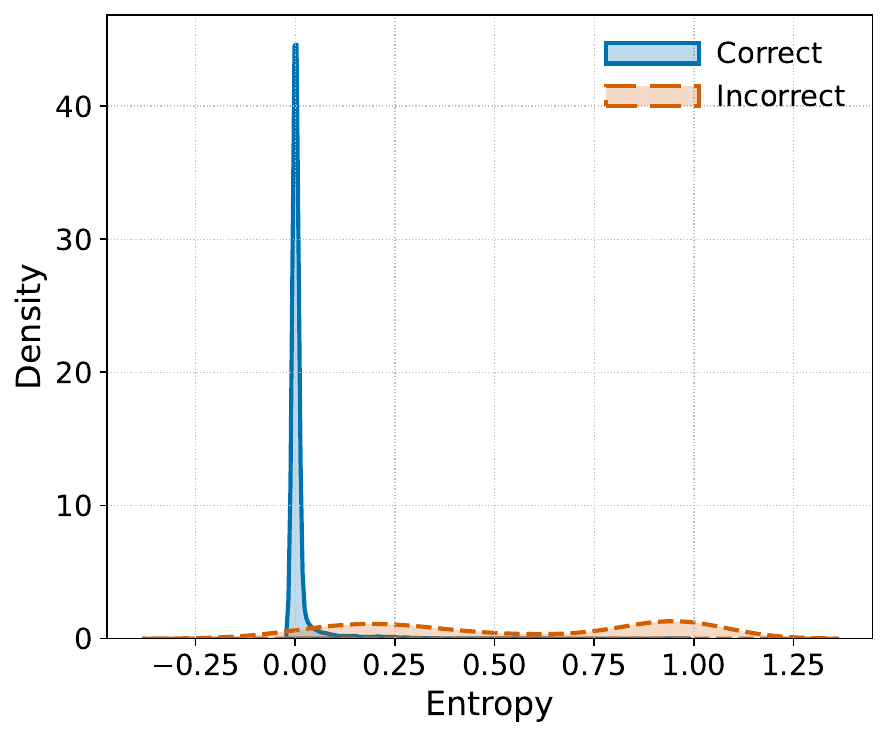}\\[-0.4em]
    {\scriptsize Softmax + EDL-CE}
\end{minipage}
\hfill
\begin{minipage}{0.30\linewidth}
    \centering
    \includegraphics[width=\linewidth]{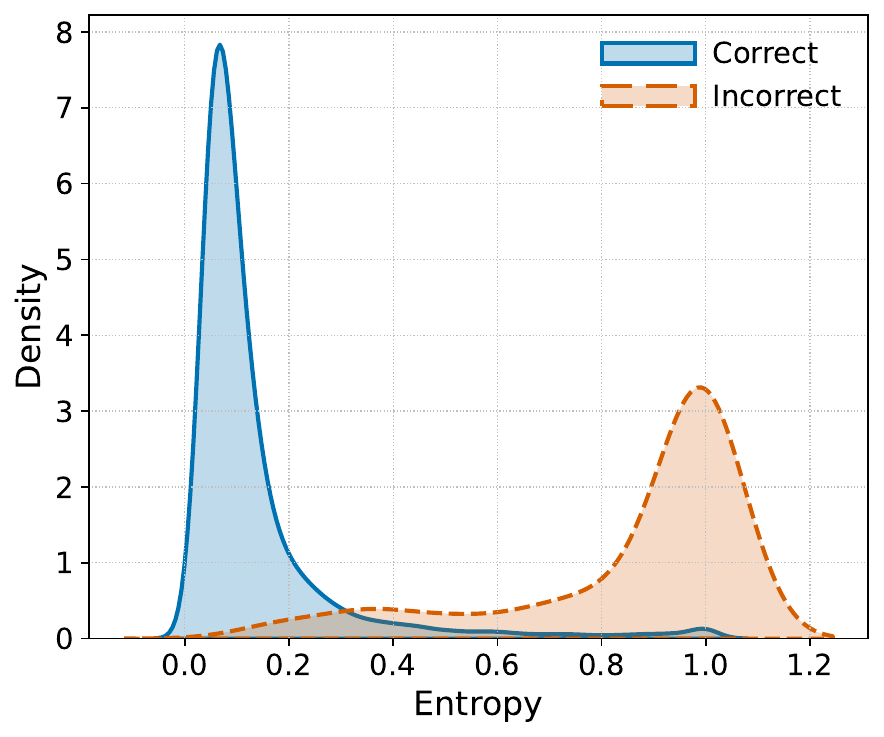}\\[-0.4em]
    {\scriptsize EDL-CE}
\end{minipage}
\hfill
\begin{minipage}{0.30\linewidth}
    \centering
    \includegraphics[width=\linewidth]{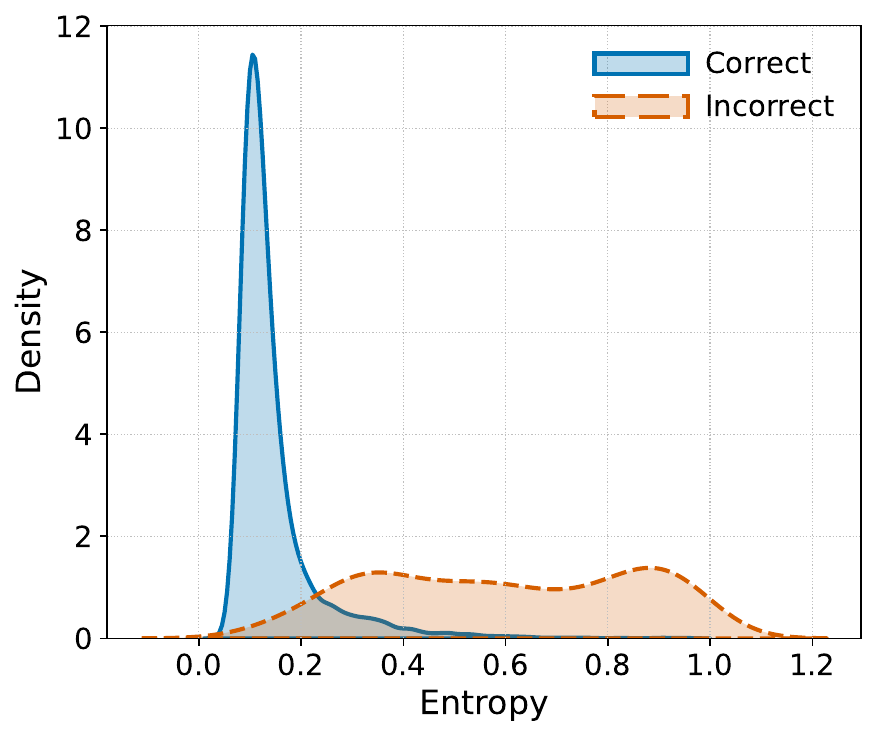}\\[-0.4em]
    {\scriptsize EDL-CE no KL}
\end{minipage}

\vspace{0.15em}

\begin{minipage}{0.30\linewidth}
    \centering
    \includegraphics[width=\linewidth]{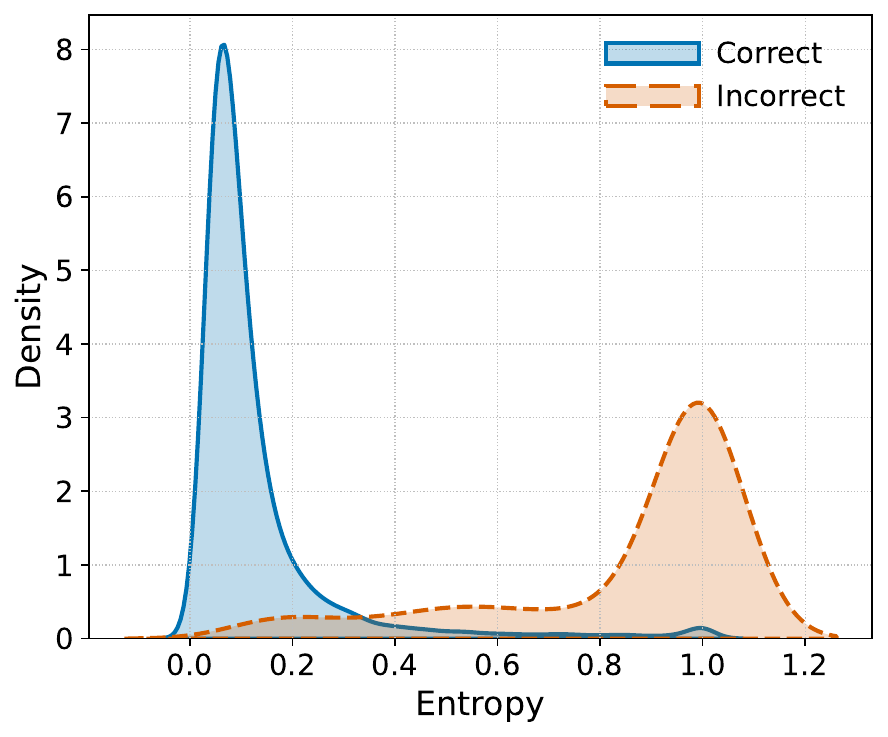}\\[-0.4em]
    {\scriptsize Plug-in EDL-CE}
\end{minipage}
\hfill
\begin{minipage}{0.30\linewidth}
    \centering
    \includegraphics[width=\linewidth]{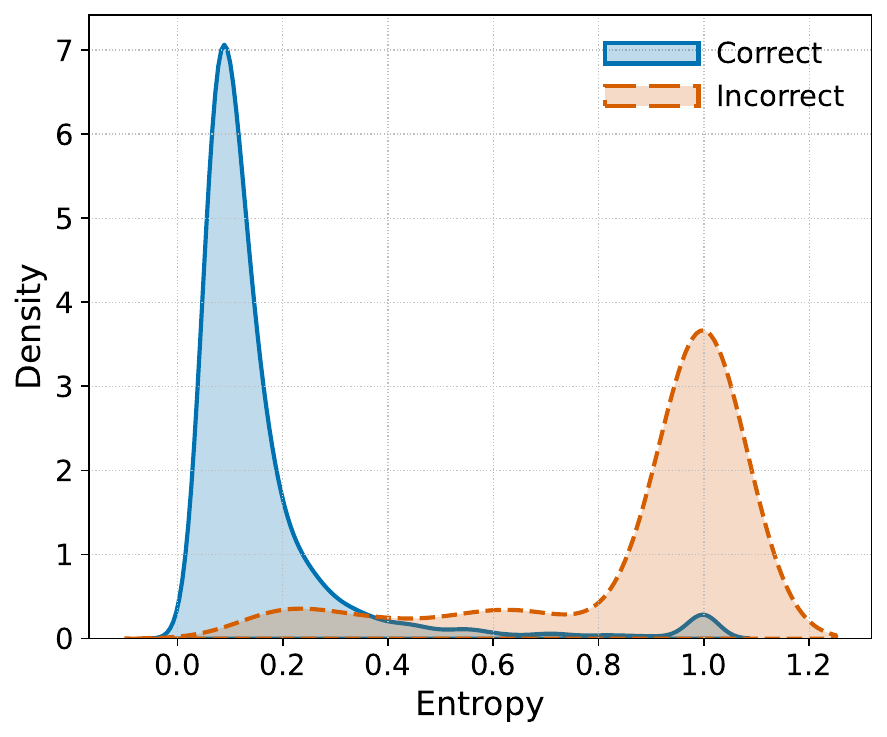}\\[-0.4em]
    {\scriptsize EDL-MSE}
\end{minipage}
\hfill
\begin{minipage}{0.30\linewidth}
    \centering
    \includegraphics[width=\linewidth]{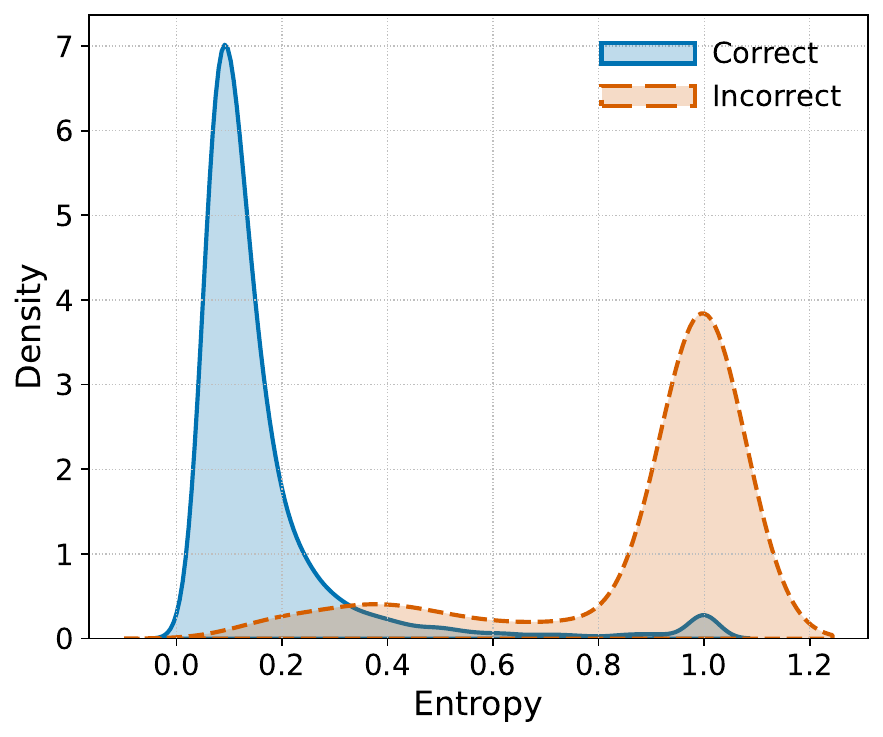}\\[-0.4em]
    {\scriptsize Plug-in EDL-MSE}
\end{minipage}
\caption{
Entropy KDE plots for all model variants on GSC V1. Each plot shows the
distribution of normalized predictive entropy for correctly and incorrectly
classified test samples. These plots are intended as diagnostic visualizations.
}
\label{fig:appendix_entropy_kde_all}

\end{figure}

\begin{figure}[!httbp]
\centering
\begin{minipage}{0.30\linewidth}
    \centering
    \includegraphics[width=\linewidth]{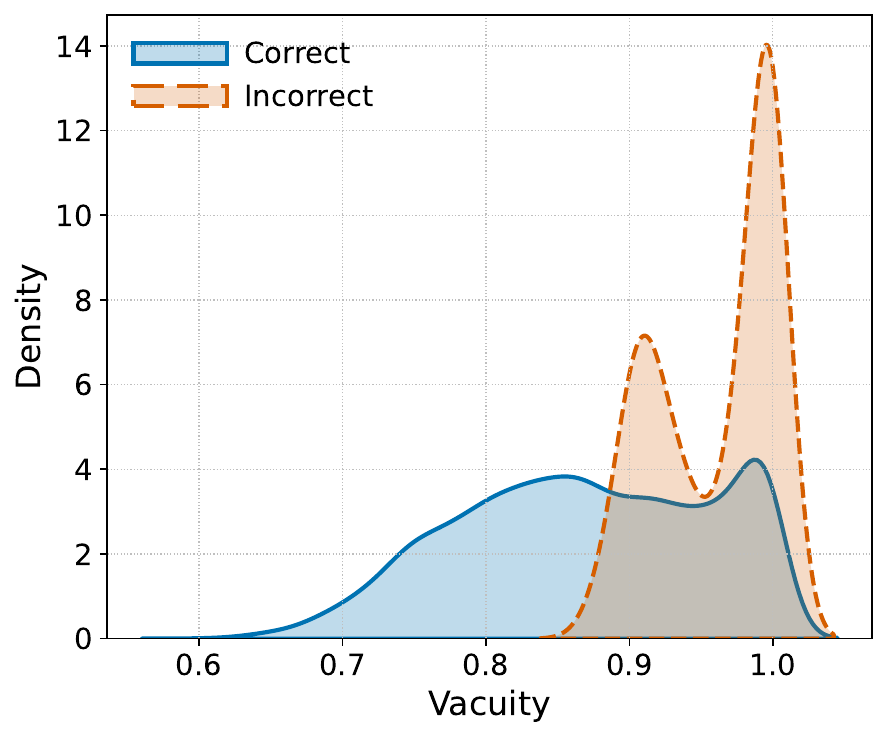}\\[-0.4em]
    {\scriptsize Softplus Proj.}
\end{minipage}
\hfill
\begin{minipage}{0.30\linewidth}
    \centering
    \includegraphics[width=\linewidth]{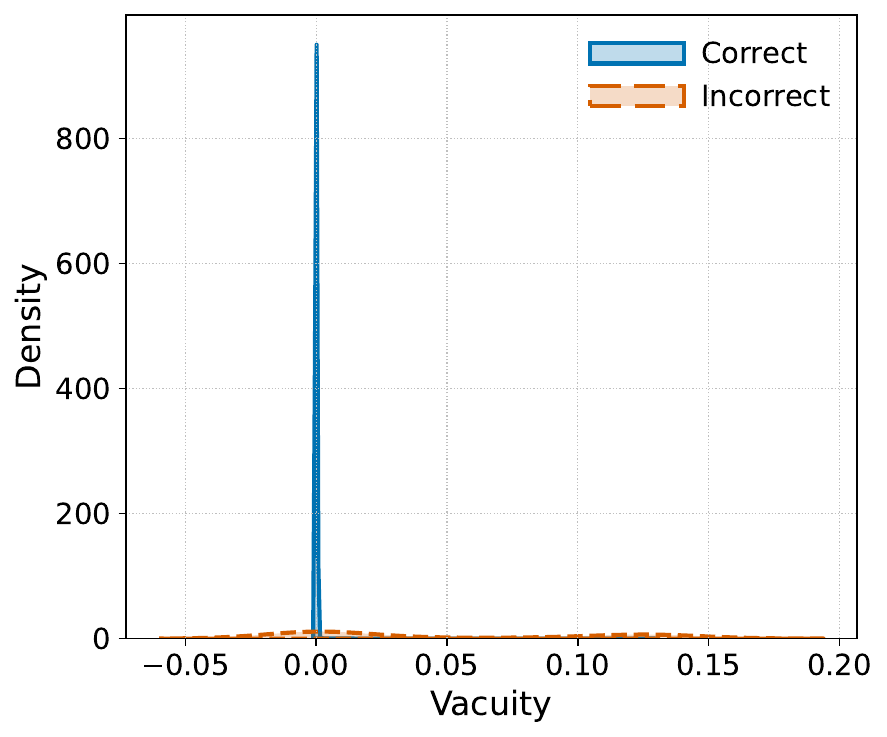}\\[-0.4em]
    {\scriptsize Softmax + EDL-CE}
\end{minipage}
\hfill
\begin{minipage}{0.30\linewidth}
    \centering
    \includegraphics[width=\linewidth]{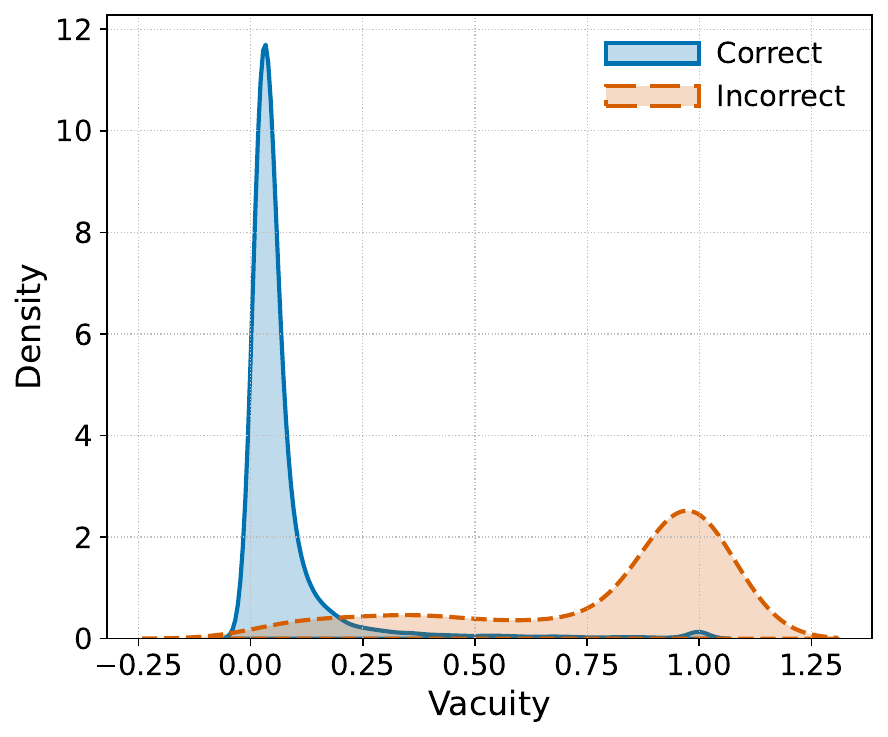}\\[-0.4em]
    {\scriptsize Plug-in EDL-CE}
\end{minipage}

\vspace{0.15em}

\hspace*{0.17\linewidth}
\begin{minipage}{0.30\linewidth}
    \centering
    \includegraphics[width=\linewidth]{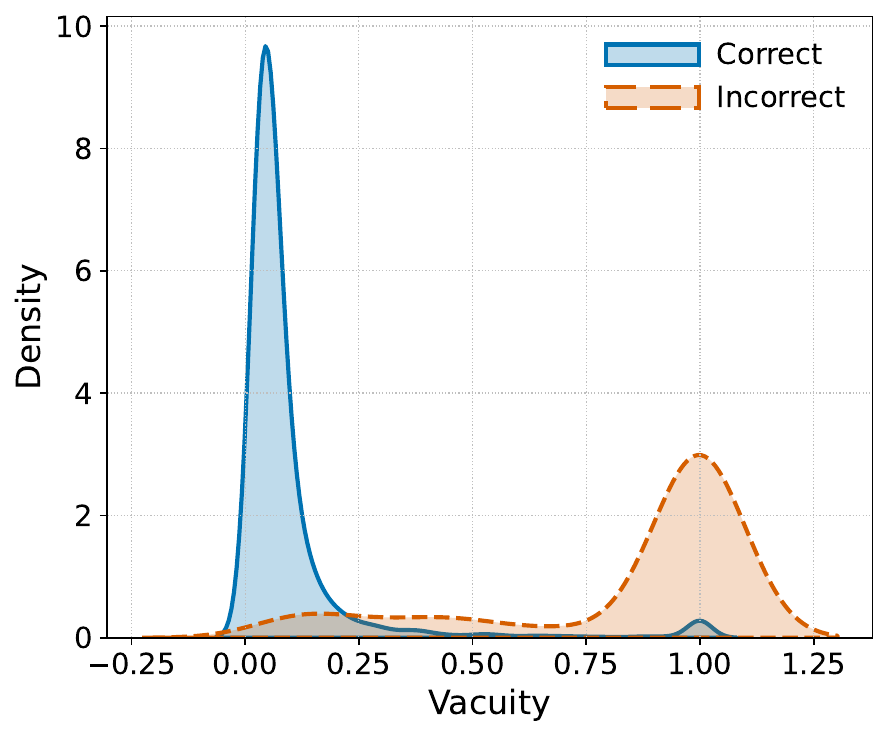}\\[-0.4em]
    {\scriptsize EDL-MSE}
\end{minipage}
\hfill
\begin{minipage}{0.30\linewidth}
    \centering
    \includegraphics[width=\linewidth]{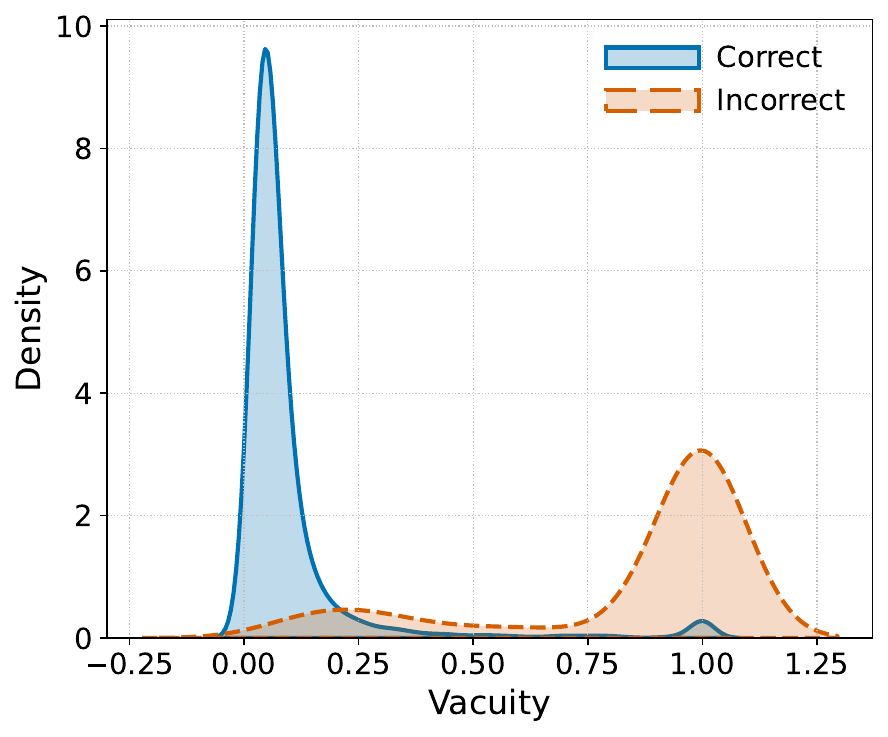}\\[-0.4em]
    {\scriptsize Plug-in EDL-MSE}
\end{minipage}
\hspace*{0.17\linewidth}
\caption{
Additional vacuity KDE plots for model variants not shown in the main text.
Each plot shows the distribution of vacuity for correctly and incorrectly
classified test samples. Together with the main-text vacuity KDEs, these plots
provide a visual diagnostic of how concentration-based uncertainty separates
correct and incorrect predictions.
}
\label{fig:appendix_vacuity_kde_remaining}

\end{figure}

\end{document}